\documentclass[11pt]{article}
\usepackage{textcomp}
\usepackage{xspace}

\usepackage{sustyle}
\usepackage{microtype}
\usepackage{graphicx}
\usepackage{subfigure}
\usepackage{booktabs} % for professional tables
\usepackage[colorlinks=true,allcolors=blue,urlcolor={black}]{hyperref}
\usepackage{array}
\usepackage{bm}
\usepackage{tcolorbox}
\usepackage{url}
\usepackage{wrapfig}
\usepackage{cleveref}
\usepackage{thm-restate}

\usepackage{chngcntr} %for latex < 2018 to use counterwithin
\counterwithin{figure}{section}
\graphicspath{{./image/}}

\newcommand{\bP}{P}
\newcommand{\bQ}{Q}
\newcommand{\var}{\mathbb{V}\text{ar}}
\newcommand{\Ep}{\mathbb{E}_{P}}
\newcommand{\Eq}{\mathbb{E}_{Q}}
\newcommand{\Kap}{\bm{\kappa}}
\newcommand{\ktilde}{\tilde{\kappa}}
\newcommand{\Ktilde}{\bm{\ktilde}}
\newcommand{\Ftilde}{\widetilde{F}}

\newcommand{\abs}[1]{\left|#1\right|}
\newcommand{\set}[1]{\left\{#1\right\}}

\newcommand{\eps}{\varepsilon}
\newcommand{\sample}{\texttt{Sample}_p}
\newcommand{\citep}{\cite}
\newcommand{\ew}{\mathtt{EW}}
\newcommand\diff{\mathrm{d}}
\newcommand{\EW}{\texttt{Edgeworth}\xspace}
\newcommand{\clt}{\texttt{CLT}\xspace}
\newcommand{\num}{\texttt{Numerical}\xspace}

\title{Sharp Composition Bounds for Gaussian Differential Privacy via Edgeworth Expansion}

\author{Qinqing Zheng\thanks{Department of Statistics. Email: {\tt zhengqinqing@gmail.com}.}
	\and Jinshuo Dong\thanks{Graduate Group in Applied Mathematics and Computational Science. Email: {\tt jinshuo@sas.upenn.edu}. }
	\and Qi Long\thanks{Department of Biostatistics, Epidemiology and Informatics. Email: {\tt qlong@pennmedicine.upenn.edu}.}
	\and Weijie J.~Su\thanks{Department of Statistics. Email: {\tt suw@wharton.upenn.edu}.}
	}
\date{}

\begin{document}
\maketitle

{\centering
\vspace*{-0.5cm}
\textit{University of Pennsylvania}\\
\par\bigskip
February 27, 2020
\par
}

\begin{abstract}
Datasets containing sensitive information are often sequentially analyzed by many
algorithms. This raises a fundamental question in differential privacy
regarding how the overall privacy bound degrades under composition. To address this question,
we introduce a family of analytical and sharp privacy bounds under composition using the
Edgeworth expansion in the framework of the recently proposed $f$-differential privacy.
In contrast to the existing composition theorems using the central limit theorem, our
new privacy bounds under composition gain improved tightness by leveraging the refined
approximation accuracy of the Edgeworth expansion. Our approach is easy to implement and
computationally efficient for any number of compositions. The superiority of these new
bounds is confirmed by an asymptotic error analysis and an application to quantifying
the overall privacy guarantees of noisy stochastic gradient descent used in training
private deep neural networks.
\end{abstract}

\section{Introduction}
Machine learning, data mining, and statistical analysis are widely applied to various
applications impacting our daily lives. While we celebrate the benefits brought by these
applications, to an alarming degree, the algorithms are accessing datasets containing
sensitive information such as individual behaviors on the web and health records. By
simply tweaking
the datasets and leveraging the output of algorithms, it is possible for an
adversary to learn information about and even identify certain individuals
\citep{fredrikson2015model,shokri2017membership}. In particular, privacy concerns become
even more acute when the same dataset is probed by a sequence of algorithms. With
knowledge of the dataset from the prior algorithms' output, an adversary can adaptively
analyze the dataset to cause additional privacy loss at each round. This reality raises
one of the most fundamental problems in the area of private data analysis:
%\\%[0.3em]
%{\centering \em How to accurately and efficiently quantify the cumulative privacy loss under composition of private algorithms?}
\vspace{-5pt}
\begin{center}
    {\em How can we accurately and efficiently quantify the cumulative\\ privacy loss under composition of private algorithms?}
\end{center}
\vskip5pt

To address this important problem, one has to start with a formal privacy definition. To
date, the most popular statistical privacy definition is $(\epsilon,
\delta)$-differential privacy (DP) \citep{Dwork2006, Dwork2006_2}, with numerous
deployments in both industrial applications and academic research~\citep{erlingsson2014rappor,abadi2016deep,papernot2016semi,ding2017collecting,apple2017learning,abowd2018us}.
Informally, this privacy definition requires an unnoticeable change in a (randomized)
algorithm's output due to the replacement of any individual in the dataset. More
concretely, letting $\epsilon \ge 0$ and $0 \le \delta \le 1$, an algorithm $M$ is
$(\eps, \delta)$-differentially private if for any pair of neighboring datasets $S$,
$S'$ (in the sense that $S$ and $S'$ differs in only one individual) and any event $E$,
\begin{equation}\label{eq:epsilon_delta}
\P(M(S) \in E) \leq e^\eps \P(M(S') \in E) + \delta.
\end{equation}

Unfortunately, this privacy definition comes with a drawback when handling composition.
Explicitly, $(\epsilon, \delta)$-DP is \textit{not closed} under composition in the
sense that the overall privacy bound of a sequence of (at least two)
$(\epsilon,\delta)$-DP algorithms cannot be \text{precisely} described by a single pair
of the parameters $\epsilon, \delta$ \citep{kairouz2017composition}. Although the
precise bound can be collectively represented by infinitely many pairs of $\epsilon,
\delta$, \cite{complexity} shows that it is computationally hard to find such a pair of privacy parameters.

The need for a better treatment of composition has triggered a surge of interest in
proposing generalizations of $(\epsilon, \delta)$-DP, including divergence-based
relaxations \citep{concentrated,concentrated2,mironov2017renyi,tcdp} and, more recently,
a hypothesis testing-based extension termed $f$-differential privacy ($f$-DP)~\citep{dong2019gdp}. This privacy definition leverages the hypothesis testing
interpretation of differential privacy, and
characterizes the privacy guarantee using
the trade-off between type I and type II errors given by the associated hypothesis testing problem. As an advantage over the divergence-based privacy definitions, among others, $f$-DP allows for a concise and sharp argument for  privacy
amplification by subsampling. More significantly, $f$-DP is accompanied with a technique
powered by the central limit theorem (CLT) for analyzing privacy bounds under composition of
a \text{large} number of private algorithms. Loosely speaking, the overall privacy bound
asymptotically converges to the trade-off function defined by testing between two normal
distributions. This class of trade-off functions gives rise to \emph{Gaussian differential privacy} (GDP), a subfamily of $f$-DP guarantees.

In deploying differential privacy, however, the number of private algorithms under
composition may be \textit{moderate} or \textit{small} (see such applications in private
sparse linear regression \citep{kifer2012private} and personalized online advertising
\citep{lindell2011practical}). In this regime, the CLT phenomenon does not kick in and,
as such, the composition bounds developed using CLT can be inaccurate
\citep{dong2019gdp,bu2019deep}. To address this practically important problem, in this
paper we develop sharp and analytical composition bounds in $f$-DP without assuming a
larger number of algorithms, by leveraging the Edgeworth expansion~\citep{hall2013bootstrap}. The Edgeworth expansion is a technique for approximating probability
distributions in terms of their cumulants. Compared with the CLT approach, in our setting, this technique enables a significant reduction of approximation errors for composition theorems.

In short, our Edgeworth expansion-powered privacy bounds have a number of appealing properties, which will be shown in this paper through both theoretical analysis and numerical examples.
\begin{itemize}
    \item  The Edgeworth expansion is a more general approach that subsumes the CLT-based approximation. Moreover, our new privacy bounds tighten the composition bounds that are developed in the prior art \citep{dong2019gdp,bu2019deep}.

    \item  Our method is easy to implement and computationally efficient. In the case where all trade-off functions are identical under composition, the computational cost is constant regardless of the number of private algorithms. This case is not uncommon and can be found, for example, in the privacy analysis of noisy stochastic gradient descent (SGD) used in training deep neural networks.
\end{itemize}

The remainder of the paper is organized as follows. Section~\ref{sec:preliminary} reviews the
$f$-DP framework.  Section~\ref{sec:edgeworth} introduces our methods based on the
Edgeworth expansion.  Section~\ref{sec:param} provides a two-parameter interpretation of the Edgeworth approximation-based privacy guarantees.  Finally, we present experimental results to demonstrate the
superiority of our approach in Section~\ref{sec:expr}.

% Section~\ref{sec:numerical} presents a numerical method that computes the composition
% bound directly, which serves as a reference in our empirical comparison with the CLT
% approximation.
% \hskip10pt $\bullet$ \hskip5pt The Edgeworth expansion is a more general approach that subsumes CLT-based approximation.
%
% \hskip10pt $\bullet$ \hskip5pt Our privacy bounds tightens the composition bound comparing to the prior arts \citep{dong2019gdp,bu2019deep}.
%
% \hskip10pt $\bullet$ \hskip5pt Our method is easy-to-implement and computationally efficient.
% When the component trade-off functions are identical, which is a common case
% including the privacy analysis for noisy SGD in training deep neural networks,
% the computational cost is constant, independent of the number of compositions.

\section{Preliminaries on $f$-Differential Privacy}
\label{sec:preliminary}

Let a randomized algorithm $M$ take a dataset $S$ as input. Leveraging the output of this algorithm, differential privacy seeks to measure the difficulty of identifying the presence or absence
of any individual in $S$. The $(\eps, \delta)$-DP definition offers such a measure using the probabilities that $M$ gives the same outcome for two neighboring datasets $S$ and $S'$. A more concrete description is as follows. Let $P$ and $Q$ denote the
probability distribution of $M(S)$ and $M(S')$, respectively. To breach the privacy, in essence, an adversary performs the following hypothesis testing problem:
\[
    H_0: \text{output} \, \sim P
    \;\; \text{vs} \;\;
    H_1: \text{output} \, \sim Q.
\]
The privacy guarantee of $M$ boils down to the extent to which the adversary can tell the two distributions apart. In the case of $(\epsilon, \delta)$-DP, the privacy guarantee is expressed via \ref{eq:epsilon_delta}. The relationship between differential privacy and hypothesis testing is first studied in \citep{wasserman2010statistical, kairouz2017composition, liu2019investigating,
balle2019hypothesis}. More recently, \cite{dong2019gdp} proposes to use the trade-off
between type I and type II errors of the optimal likelihood ratio tests at level ranging from 0 to 1 as a measure of the privacy guarantee. Note that the optimal tests are given by the Neyman--Pearson lemma, and can be thought of as the most powerful adversary.

%, as the measure ofhardness.

\paragraph{Trade-off function.}
Let $\phi$ be a rejection rule for testing against $H_0$ against $H_1$. The type I and type II error of $\phi$ are $\E_P(\phi)$ and $1 - \E_Q(\phi)$, respectively.
The trade-off function $T: [0,1] \rightarrow [0,1]$
between the two probability distributions $P$ and $Q$ is defined as
\[
T(P, Q)(\alpha) = \inf_{\phi} \set{1 - \E_Q(\phi): \E_P(\phi) \leq \alpha}.
\]
That is, $T[P,Q](\alpha)$ equals the minimum type II error that one can achieve at significance level $\alpha$.
A larger trade-off function corresponds to a more difficult hypothesis testing problem, thereby implying more privacy of the associated private algorithm. When the two distributions are the same, the perfect privacy is achieved and
the corresponding trade-off function is $T(P,P)(\alpha) = 1 - \alpha$. In the sequel, we denote this function by $\text{Id}(\alpha)$. With the definition of trade-off functions in place, \cite{dong2019gdp} introduces the following privacy definition (we say $f \ge g$ if $f(\alpha) \ge g(\alpha)$ for all $0 \le \alpha \le 1$):
%The authors then is motivated to define the following privacy notion.
\begin{definition}\label{def:f-dp}
Let $f$ be a trade-off function. An algorithm $M$ is $f$-differentially private if $T(M(S), M(S')) \geq f$ for any pair of neighboring datasets $S$ and $S'$.
\end{definition}

While the definition above considers a general trade-off function, it is worthwhile
noting that $f$ can always be assumed to be \textit{symmetric}. Letting $f^{-1}(\alpha) := \inf\{0 \leq t \leq 1: f(t) \leq \alpha\}$ (note that $f^{-1} = T(Q, P)$ if $f = T(P, Q)$), a trade-off function $f$ is said to be symmetric if $f = f^{-1}$. Due to the symmetry of the two neighboring datasets in the privacy definition, an $f$-DP algorithm must be $\max\{f, f^{-1}\}$-DP. Compared to $f$, the new trade-off function $\max\{f, f^{-1}\}$ is symmetric and gives a greater or equal privacy guarantee. For the special case where
the lower bound in Definition~\ref{def:f-dp} is a trade-off function between two Gaussian distributions, we say that the algorithm has \emph{Gaussian differential privacy} (GDP):
\begin{definition}
Let $G_\mu := T(\N(0,1), \N(\mu, 1)) \equiv \Phi(\Phi^{-1}(1-\alpha) - \mu)$ for some
$\mu \geq 0$, where $\Phi$ denotes the cumulative distribution function (CDF) of the
standard normal distribution. An algorithm $M$ gives $\mu$-GDP if $T(M(S), M(S')) \geq G_\mu$ for any pair of neighboring datasets $S$ and $S'$.
\end{definition}

\begin{figure}[t]
    \centering
    \includegraphics[width=0.45\columnwidth]{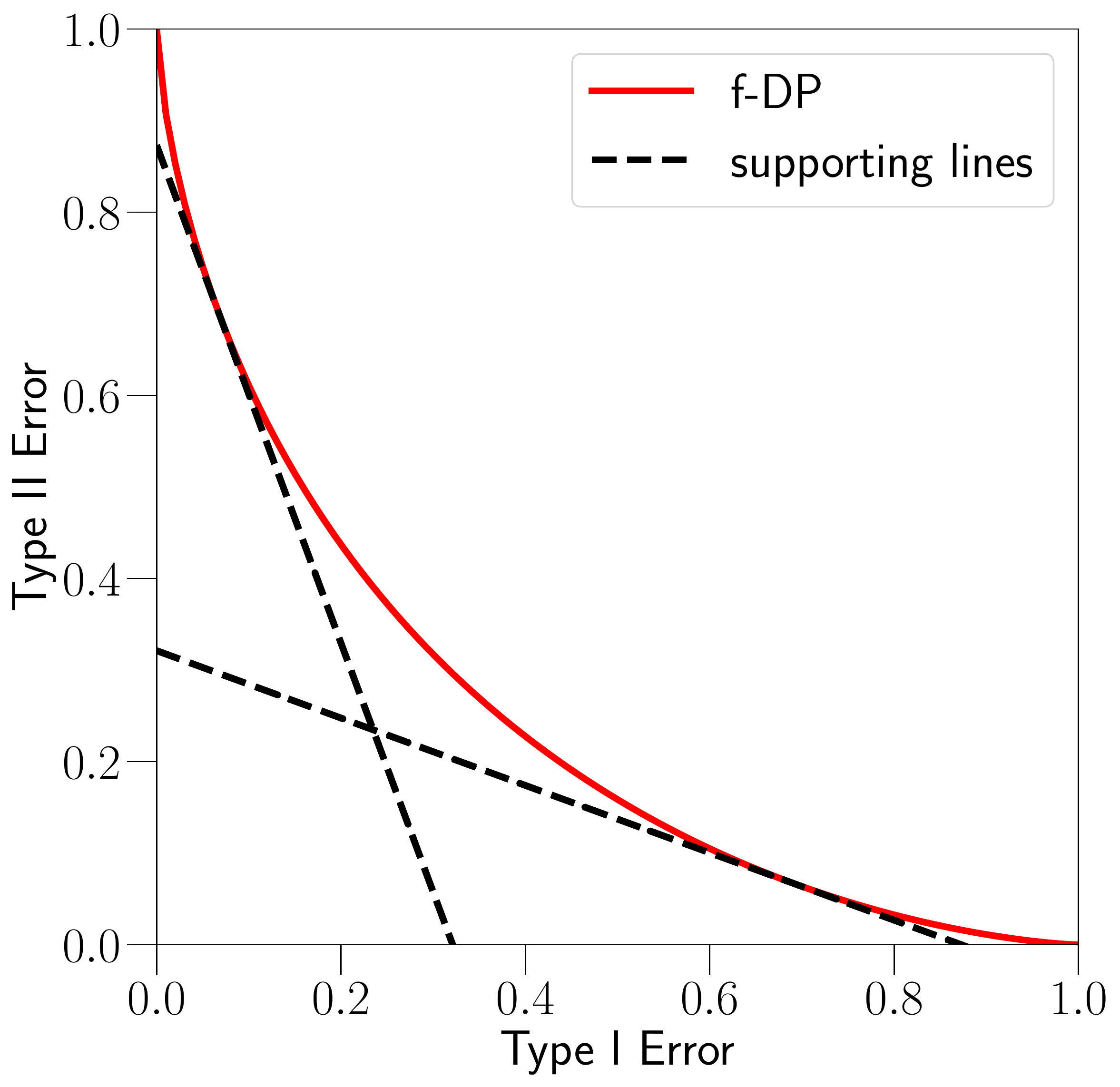}
    \vskip-5pt\caption{
    An example of a trade-off function and two supporting lines induced by the associated $(\eps, \delta(\eps))$-DP guarantees. These lines
    have slopes $-e^{\pm \eps}$, respectively, and intercepts $1-\delta$.
    }
    \label{fig:primal_dual}
\end{figure}
\paragraph{Duality to $(\eps, \delta)$-DP.}
The $f$-DP framework has a dual relationship with $(\eps,
\delta)$-DP in the sense that $f$-DP is equivalent to an infinite collection of $(\eps,
\delta)$-DP guarantees via the convex conjugate of $f$.
One can view $f$-DP as the primal representation of privacy, and accordingly, its
dual representation is the collection of $(\eps, \delta)$-DP guarantees.
In this paper, the Edgeworth approximation addresses
$f$-DP from the primal perspective. However, it is also instructive to check the
dual presentation. The following propositions introduce how to convert the primal to the dual, and vice versa. Geometrically, each associated $(\eps, \delta)$-DP guarantee defines two symmetric supporting linear functions to $f$ (assuming $f$ is symmetric). See Figure~\ref{fig:primal_dual}.
\begin{proposition}[Primal to Dual]
\label{prop:primaldual}
    For a symmetric trade-off function $f$, let $f^*: \R \rightarrow \R$ be its convex conjugate function
    $f^*(y)=\sup_{0 \leq x \leq 1} yx - f(x)$.  A mechanism is $f$-DP if and only if it is
    $(\eps, \delta(\eps))$-DP for all $\eps \geq 0$ with $\delta(\eps) = 1 + f^*(-e^\eps)$.
\end{proposition}
\begin{proposition}[Dual to Primal]
\label{prop:dualprimal}
    Let $I$ be an arbitrary index set such that each $i \in I$
is associated with $\eps_i \in [0, \infty)$ and $\delta_i \in [0, 1]$.
A mechanism is $(\eps_i, \delta_i)$-DP
for all $i \in I$ if and only if it is $f$-DP with $f = \sup_{i \in I} f_{\eps_i, \delta_i}$, where
$f_{\eps, \delta}(\alpha) = \max\set{0, 1 -\delta - e^{\eps}\alpha, e^{-\eps}(1 - \delta - \alpha)}$ is the trade-off function corresponding to $(\epsilon, \delta)$-DP.
\end{proposition}

%% as the later update obtains the information learnt by previous updates. In

Next, we introduce how $f$-DP guarantees degrade under composition. With regard to composition, SGD offers an important benchmark for testing a privacy definition. As a popular optimizer for training deep neural networks, SGD outputs a series of models that are generated from the \emph{composition} of many gradient descent updates. Furthermore, each step of update is computed from a \emph{subsampled} mini-batch of data points. While composition degrades the privacy, in contrast, subsampling amplifies the privacy as individuals uncollected in the mini-batch have perfect privacy. Quantifying these two operations under the $f$-DP framework is crucial for analyzing the privacy guarantee of deep learning models trained by noisy SGD.

\paragraph{Composition.}
Let $f_1 = T(P_1, Q_1)$ and $f_2 = T(P_2, Q_2)$, \cite{dong2019gdp} defines a binary
operator $\otimes$ on trade-off functions such that $f_1 \otimes f_2 = T(P_1 \times P_2,
Q_1 \times Q_2)$, where $P_1 \times P_2$ is the distribution product.  This operator
is commutative and associative. The composition primitive refers to an algorithm $M$ that consists of $n$ algorithms $M_1, \ldots, M_n$, where $M_i$ observes
both the input dataset and output from all previous algorithms\footnote{In this paper, $n$ denotes the number of private algorithms under composition, as opposed to the number of individuals in the dataset. This is to be consistent with the literature on central limit theorems.}. In \cite{dong2019gdp}, it is shown that if $M_i$ is $f_i$-DP for $1 \le i \le n$, then the composed algorithm $M$ is $f_1
\otimes \cdots \otimes f_n$-DP.  The authors further identify a central limit theorem-type phenomenon of the overall privacy loss under composition. Loosely speaking, the privacy guarantee asymptotically converges to GDP in the sense that $f_1 \otimes \cdots \otimes f_n \rightarrow G_\mu$ as
$n \rightarrow \infty$ under certain conditions. The privacy parameter $\mu$ depends on the trade-off functions $f_1, \ldots, f_n$.

\paragraph{Subsampling.} Consider the operator $\sample$ that includes each individual in the
dataset with probability $p$ independently. Let $M \circ \sample$ denote the
algorithm where $M$ is applied to the subsampled dataset. In the subsampling theorem for $f$-DP, \cite{dong2019gdp} proves that
if $M$ is $f$-DP, then $M\circ \sample(S)$ is $\tilde f$-DP if $\tilde f \le f_p$ and $\tilde f \le f_p^{-1}$, where $f_p = p f +
(1-p)\text{Id}$. As such, we can take $\tilde f = \min\{ f_p, f_p^{-1} \}$, which is not
convex in general however. This issue can be resolved by using $\min\{ f_p, f_p^{-1} \}^{**}$ in place of $\min\{ f_p, f_p^{-1} \}$, where $f^{**}$ denotes the double conjugate of $f$. Indeed, \cite{dong2019gdp} shows that the subsampled algorithm is $\min\set{f_p, f_p^{-1}}^{**}$-DP.

%%,which is the biconjugate of $\min\set{f_p, f^{-1}_p}$.

\paragraph{Noisy SGD.} Let $M_\sigma$ denote the noisy gradient descent update, where
$\sigma$ is the scale of the Gaussian noise added to the gradient. The noisy SGD update can essentially be represented as $M_\sigma \circ
\text{sample}_p$.
Exploiting the above results for composition and
subsampling, \cite{bu2019deep} shows that $M_\sigma \circ
\text{Sample}_p$ is $\min\{h, h^{-1}\}$-DP,
where $h = p G_{1/\sigma} + (1-p)\text{Id}$. Recognizing that noisy SGD with $n$ iterations is the $n$-fold composition of $M_\sigma \circ
\text{Sample}_p(S)$, the overall privacy lower bound is $\min\{g, g^{-1}\}$-DP, where $g = h^{\otimes n} = (p G_{1/\sigma} +
(1-p)\text{Id})^{\otimes n}$. To evaluate the composition bound, \cite{bu2019deep} uses a central limit theorem-type result in the asymptotic regime where $p\sqrt{n}$ converges
to a positive constant as $n \rightarrow \infty$: in this regime, one can show $g \rightarrow
G_{p\sqrt{n(e^{1/\sigma^2} - 1)}}$ and consequently $\min\set{g, g^{-1}}^{**} \rightarrow G_{p\sqrt{n(e^{1/\sigma^2} - 1)}}$ as well.

%%As mentioned earlier, $p G_{1/\sigma} + (1-p)\text{Id} $ is asymmetric and so is the function $g$, and we symmetrize it to obtain the final privacy bound $\min\set{g, g^{-1}}^{**}$.

\section{Edgeworth Approximation}
\label{sec:edgeworth}
In this section, we introduce the Edgeworth expansion-based approach to computing the privacy bound under composition. The development of this approach builds on \cite{dong2019gdp}, with two crucial modifications.

% important adjustments.

%%Let $\displaystyle L_i(\cdot) = \log \frac{q_i(\cdot) }{p_i(\cdot)}$, the likelihood ratio test statistic is given by

Consider the hypothesis testing problem
$H_0: \bx \sim P_1 \times \cdots \times P_n$
vs
$H_1: \bx \sim Q_1 \times \cdots \times Q_n$.
Let $\bP$ denote the distribution $P_1 \times \cdots \times P_n$, and
$p_i(\cdot)$ denote the probability density functions of $P_i$. Correspondingly,
we define $\bQ$ and $q_i$ in the same way. Letting $\displaystyle L_i(\cdot) = \log \frac{q_i(\cdot) }{p_i(\cdot)}$, the likelihood ratio test statistic is given by
$\displaystyle
    T_n(\bx) = \log \dfrac{\prod_{i=1}^n q_i(x_i)}{\prod_{i=1}^n p_i(x_i)} = \sum_{i=1}^n L_i(x_i).
$
The Neyman--Pearson lemma states that the most powerful test at a given significance level
$\alpha$ must be a thresholding function of $T_n(\bx)$.
As a result, the optimal rejection rule would reject
$H_0$ if $ T_n(\bx) > \eta $, where $\eta$ is determined by $\alpha$. An equivalent rule is to apply
thresholding to the standardized statistic: $H_0$ is rejected if
\begin{equation}
    \label{eq:rule}
    \frac{ T_n(\bx) - \Ep[T_n(\bx)] }{ \sqrt{ \var_{\bP}[ T_n(\bx) ] } } > h(\alpha),
\end{equation}
where the threshold $h$ is determined by $\alpha$.

%%for notation simplicity
%%cumulative distributed function
%%By Lyapunov CLT

In the sequel, for notational simplicity we shall use $T_n$ to denote $T_n(\bx)$, though it is a function of $\bx$.  Let $F_n(\cdot)$ be the CDF of $\dfrac{T_n - \Ep[T_n] }{ \sqrt{ \var_{\bP}[T_n] } }$ when $\bx$ is drawn
from $\bP$. That is,
$
    F_n(h) = \P_p \left( \frac{ T_n - \Ep[T_n] }{ \sqrt{ \var_{\bP}[ T_n ] } } \leq h \right).
$
By the Lyapunov CLT, the standardized statistic
$ \frac{ T_n - \Ep[T_n] }{ \sqrt{ \var_{\bP}[ T_n ] } } $
converges in distribution
to the standard normal random variable. In other words, it holds that
\[
    \frac{ T_n - \Ep[T_n] }{ \sqrt{ \var_{\bP}[ T_n ] } } \xrightarrow{d}
    \N(0,1), \hskip15pt F_n(\cdot) \rightarrow \Phi(\cdot)
\]
as $n \goto \infty$. Likewise, we write
$
    \Ftilde_n(h) = \P_Q \left(\frac{ T_n - \Eq[T_n] }{ \sqrt{ \var_{\bQ}[ T_n ]
} } \leq h \right)
$
with $\bx \sim \bQ$ and get
\[
    \frac{ T_n - \Eq[T_n] }{ \sqrt{ \var_{\bQ}[ T_n ] } } \xrightarrow{d}
\N(0,1), \hskip15pt \Ftilde_n(\cdot) \rightarrow \Phi(\cdot).
\]
With these notations in place, one can write the type I error of the rejection rule (\ref{eq:rule}) as
\begin{equation}
\label{eq:type_1}
    \alpha  =  \P_P \bigg( \frac{ T_n - \Ep[T_n] }{ \sqrt{ \var_{\bP}[ T_n ] } } > h  \bigg)
            = 1 - F_n(h).
\end{equation}
The type II error of this test, which is $f(\alpha)$ by definition, is given by
\[
\begin{aligned}
    f(\alpha) & = && \P_Q \bigg( \frac{ T_n - \Ep[T_n] }{ \sqrt{ \var_{\bP}[ T_n ] } } \leq  h  \bigg) \\
              & = && \P_Q \bigg( \frac{ T_n - \Eq[T_n] }{ \sqrt{ \var_{\bP}[ T_n ] } } \leq
h - \frac{\Eq[T_n] - \Ep[T_n]}{\sqrt{\var_{\bP}[T_n]}} \bigg).
\end{aligned}
\]
In \cite{dong2019gdp}, the authors assume that $f$ is symmetric and therefore derive the identity $\var_{\bQ}(T_n) = \var_{\bP}(T_n)$. As a consequence, $f(\alpha)$ can be written as $\Ftilde_n(h - \mu_n)$, where $\mu_n = (\Eq[T_n] - \Ep[T_n])/\sqrt{\var_{\bQ}[T_n]}$. Taken together, the equations above give rise to $f(\alpha) = \Ftilde_n(F^{-1}_n(1 - \alpha) - \mu_n)$. Leveraging this expression of $f$, \cite{dong2019gdp} proves a CLT-type asymptotic convergence result under certain conditions:
\begin{equation}
    \label{eq:clt}
    f(\alpha) \rightarrow G_\mu(\alpha) = \Phi(\Phi^{-1}(1 - \alpha) - \mu)
\end{equation}
as $n \goto \infty$, where $\mu$ is the limit of $\mu_n$.

Now, we discard the symmetry assumption and just rewrite
%\scalebox{0.9}{
%\begin{minipage}{\columnwidth}
\begin{equation}
\begin{aligned}
     f(\alpha)   =  &\hskip3pt \P_Q \bigg( \frac{ T_n - \Ep[T_n] }{ \sqrt{ \var_{\bP}[ T_n ] } } \leq  h  \bigg) \\
               = &\hskip3pt \P_Q \bigg( \frac{ T_n - \Eq[T_n] }{ \sqrt{ \var_{\bP}[ T_n ] } } \leq
h - \frac{\Eq[T_n] - \Ep[T_n]}{\sqrt{\var_{\bP}[T_n]}} \bigg) \\
               =  &\hskip3pt \P_Q \bigg( \frac{ T_n - \Eq[T_n] }{ \sqrt{ \var_{\bQ}[ T_n ] } } \leq
                  \bigg[ h - \frac{\Eq[T_n] - \Ep[T_n]}{\sqrt{\var_{\bP}[T_n]}}\bigg]
              \sqrt{\frac{\var_{\bP}[T_n]}{\var_{\bQ}[T_n]}} \bigg) \\
               = &\hskip3pt \Ftilde_n\bigg( (h - \mu_n)
              \sqrt{\frac{\var_{\bP}[T_n]}{\var_{\bQ}[T_n]}} \bigg).
\end{aligned}
\label{eq:total}
\end{equation}
%\end{minipage}
%}
Plugging Equation~\ref{eq:type_1} into \ref{eq:total}, we obtain
\begin{equation}
    \label{eq:ftrue}
    f(\alpha) = \Ftilde_n\bigg( \left(F^{-1}_n(1 - \alpha) - \mu_n \right) \sqrt{\frac{\var_{\bP}[T_n]}{\var_{\bQ}[T_n]}} \bigg).
\end{equation}
In the special case $f$ is symmetric, the factor
$\sqrt{\frac{\var_{\bP}[T_n]}{\var_{\bQ}[T_n]}} $ is equal to one and we recover the result in \cite{dong2019gdp}.

To obtain the composition bound, the exact computing of Equation~\ref{eq:ftrue} is not
trivial. In Section~\ref{sec:numerical} we present a numerical method to compute it
directly; however, this method is computationally daunting and could not scale to a large
number of compositions. The CLT estimator (Equation~\ref{eq:clt}) can be computed
quickly, however it can be loose for a small or moderate number of compositions.  More
importantly, in practice, we observe that the CLT estimator does not handle the
composition of asymmetric trade-off functions well. To address these issues, we propose
a two-sided approximation method,  where the Edgeworth expansion is applied to both $F_n$
and $\Ftilde_n$ in Equation~\ref{eq:ftrue}.  Our method leads to more accurate
description of $f(\alpha)$, as justified in Section~\ref{sec:expr}.

\subsection{Technical Details}
In Equation~\ref{eq:ftrue}, we need to evaluate $\Ftilde_n$ and $F_n^{-1}$.
Our methods for addressing each of them are described below.
%\subsubsection{Approximate $\Ftilde_n$}
\paragraph{Approximate $\Ftilde_n$.}
\label{sec:approx_Ftilde}
Assume $\bx \sim \bQ$. Denote
\begin{equation}
    X_Q = \frac{T_n - \Eq[T_n]}{\sqrt{\var_\bQ(T_n)}} = \frac{\sum_{i=1}^n (L_i - \mu_i)}{\sqrt{\sum_{i=1}^n \sigma^2_i}},
\end{equation}
where $\mu_i$ and $\sigma_i^2$ are the mean and variance of $L_i$ under distribution
$Q_i$.  Recall $\Ftilde_n$ is the CDF of $X_Q$, and
we can apply the Edgeworth expansion to approximate $\Ftilde_n$ directly, following the techniques introduced in
\cite{hall2013bootstrap}. It provides a family of series that approximate
$\Ftilde_n$ in terms of the cumulants
of $X_Q$, which can be derived from the cumulants of $L_i$s under distribution $Q_i$s. See Definition~\ref{def:cumulant} for the notion of cumulant
and Appendix~\ref{app:cumulant} for how to compute them.

\begin{definition}
\label{def:cumulant}
For a random variable $X$, let $K(t)$ be the natural logarithm of the
moment-generating function of $X$: $K(t) = \log \E\left(e^{tX}\right)$.
% The function $K(t)$ is called the \emph{cumulant-generating function},
The cumulants of $X$, denoted by $\kappa_0, \ldots, \kappa_r, \ldots$ for integer $r>0$,
are the coefficients in the Taylor expansion of $K(t)$ about the
origin:
$
    K(t) = \log \E\left(e^{tX}\right) = \sum_{r=0}^\infty \kappa_r t^r / r!.
$
\end{definition}
\vskip-5pt
The key idea of the Edgeworth approximation is to write the characteristic function of the distribution of
$X_Q$ as the following form:
\[
\chi_n(t)=\bigg(1+\sum _{j=1}^{\infty }{\frac {P_{j}(it)}{n^{j/2}}}\bigg)\exp(-t^{2}/2),
\]
where $P_j$ is a polynomial with degree $3j$, and then truncate the series up to a fixed
number of terms. The corresponding CDF approximation is obtained by the inverse Fourier transform of
the truncated series. The Edgeworth approximation of degree $j$ amounts to truncating the above series up to terms of order $n^{-\frac{j}{2}}$.

Let $\ktilde_r(L_i)$ be the $r$-th cumulant of $L_i$ under $Q_i$, and $\Ktilde_r = \sum_{i=1}^n \ktilde_r(L_i)$.
Let $\sigma_n = \sqrt{\sum_{i=1}^n \sigma^2_i}$.
Denoted by $\Ftilde_\ew(h)$
\footnote{We note that $\Ftilde_\ew(h)$ is not
guaranteed to be a valid CDF in general, however it is a numerical approximation to $F_n(h)$
with improved error bounds compared to the CLT approximation.},
the degree-$2$ Edgeworth approximation of $\Ftilde_n(h)$ is given by
\begin{equation}
    \label{eq:approx_Ftilde}
    \begin{aligned}
       &  \Ftilde_\ew(h) = \Phi(h)
        - \overbrace{\sigma_n^{-3} \cdot \frac{1}{6}  \Ktilde_3 (h^2 - 1) \phi(h) }^{T_1}\\
        % & \hskip20pt  - \overbrace{  \sigma_n^{-4} \cdot   \frac{1}{24} \Ktilde_4 (h^3 - 3h) \phi(h) }^{T_2} \\
        % & \hskip20pt -\overbrace{  \sigma_n^{-6} \cdot \frac{1}{72} \Ktilde_3^2 (h^5 - 10h^3 + 15h) \phi(h)  }^{T_3}.
        & - \overbrace{  \left(
        \frac{  \sigma_n^{-4}    }{24} \Ktilde_4 (h^3 - 3h)
         + \frac{ \sigma_n^{-6}  }{72} \Ktilde_3^2 (h^5 - 10h^3 + 15h) \right)
         \phi(h)}^{T_2}.
    \end{aligned}
\end{equation}
The term $T_1$ is of order $n^{-\frac12}$ and $T_2$ is of order
$n^{-1}$. See Appendix~\ref{app:edgeworth} for detailed derivations.
Our framework covers the CLT approximation introduced in \cite{dong2019gdp}.
The CLT approximation is equivalent to the degree-$0$ Edgeworth approximation,
whose approximation error of order $n^{-\frac12}$.

% \subsubsection{Approximate $F^{-1}_n$}
\paragraph{Approximate $F^{-1}_n$.}
\begin{figure}[t]
    \centering
    \includegraphics[width=0.75\textwidth]{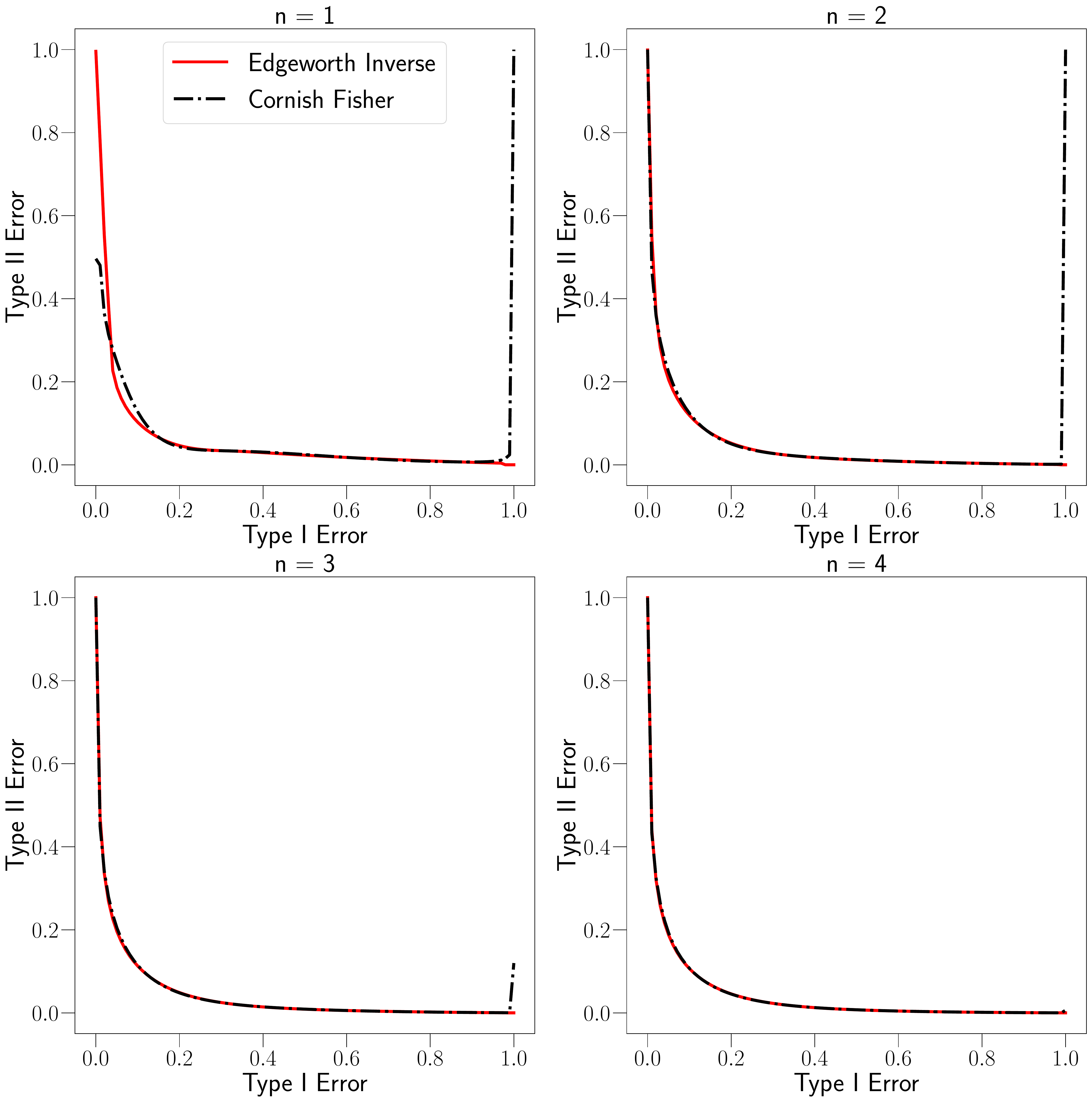}
    \vskip-5pt
    \caption{Comparison of $f(\alpha)$ obtained when approximating $F^{-1}_n(\cdot)$
        using Method~(i) and~(ii). The approximation of $\Ftilde_n(\cdot)$ is fixed
        as in Equation~\ref{eq:approx_Ftilde}.
    }
    \label{fig:cornish_fisher}
\end{figure}
In Equation~\ref{eq:ftrue} we need to compute the $F^{-1}_n(1-\alpha)$.
This is the $1-\alpha$ quantile of the distribution of
$\frac{T_n(\bx) - \Ep[T_n(\bx)]}{\sqrt{\var_\bP[T_n(\bx)]}}$, where $\bx$ is $\bP$ distributed.
We consider two approaches to deal with it.

\emph{Method~(i):} % As the same as in Section~\ref{sec:approx_Ftilde},
First compute the degree-$2$
Edgeworth approximation $F_\ew(\cdot)$ of $F_n(\cdot)$:
% Let $\kappa_r(L_i)$ be the $r$-th order cumulant of $L_i$ under $P_i$ and $\Kap_r = \sum_{i=1}^n \kappa_r(L_i)$.
\begin{equation}
    \label{eq:approx_Fn}
    \begin{aligned}
        F_\ew(h)  = & \Phi(h)  - \sigma_n^{-3} \cdot \frac{1}{6}  \Kap_3 (h^2 - 1) \phi(h) \\
         & - \sigma_n^{-4} \cdot \frac{1}{24} \Kap_4 (h^3 - 3h) \phi(h) \\
         & - \sigma_n^{-6} \cdot \frac{1}{72} \Kap_3^2 (h^5 - 10h^3 + 15h) \phi(h),
    \end{aligned}
\end{equation}
where $\Kap_r = \sum_{i=1}^n \kappa_r(L_i)$ and $\kappa_r(L_i)$ is the $r$-th cumulant of $L_i$ under $P_i$.
Next, numerically solve equation $F_\ew(h) -  1 - \alpha = 0$.

\emph{Method~(ii):} Apply the closely related \emph{Cornish-Fisher Expansion} \citep{cornish1938moments,
        fisher1960percentile},
        an asymptotic expansion used to approximate the quantiles of a probability
        distribution based on its cumulants,
        to approximate $F^{-1}_n(\cdot)$ directly.
Let $z = \Phi^{-1}(1-\alpha)$ be the $1-\alpha$ quantile of the standard normal
distribution. The degree-2 Cornish-Fisher approximation of the $1 - \alpha$ quantile of
$ \frac{T_n - \Ep[T_n]}{\sqrt{\var_\bP(T_n)}}$ is given by
\begin{equation}
\label{eq:cornish_fisher}
  \begin{aligned}
    &  F^{-1}_n(1-\alpha) \approx z + \sigma_n^{-3} \frac{1}{6} \Kap_3 (z^2 - 1) \\
    & \hskip10pt + \sigma_n^{-4} \frac{1}{24} \Kap_4 (z^3 - 3 z )  - \sigma_n^{-6} \frac{1}{36} \Kap_3^2 (2 z^3 - 5z ).
  \end{aligned}
\end{equation}
%\end{proposition}
Both approaches have pros and cons. The Cornish-Fisher approximation has closed form
solution, yet Figure~\ref{fig:cornish_fisher} shows that it is unstable
at the boundary when the number of compositions is small.
For our experiments in Section~\ref{sec:expr}, we use
the numerical inverse approach throughout all the runs.

\subsection{Error Analysis}

Here we provide an error bound for approximating the overall privacy level $f$ using
Edgeworth expansion. For simplicity, we assume that $f$ is symmetric and the
log-likelihood ratios $L_i$'s are iid distributed with the common distribution having
sufficiently light tails for the convergence of the Edgeworth expansion, under both $P$ and $Q$.

The Edgeworth expansion of degree 2 satisfies both $F_n(h) - F_{\ew}(h) = O(n^{-\frac32})$ and $\widetilde F_n(h) - \widetilde F_{\ew}(h) = O(n^{-\frac32})$. Conversely, the inverse satisfies $F_n^{-1}(\alpha) - F_{\ew}^{-1}(\alpha) = O(n^{-\frac32})$ and $\widetilde F_n^{-1}(\alpha) - \widetilde F_{\ew}^{-1}(\alpha) = O(n^{-\frac32})$ for $\alpha$ that is bounded away from 0 and 1. Making use of these approximation bounds, we get
\[
\begin{aligned}
f_{\ew}(\alpha) &= \tilde F_{\ew}(F^{-1}_{\ew}(1 - \alpha) - \mu)\\
                &= \tilde F_{\ew}(F^{-1}_{n}(1 - \alpha) + O(n^{-\frac32})- \mu)\\
                &= \tilde F_{n}(F^{-1}_{n}(1 - \alpha) + O(n^{-\frac32})- \mu) + O(n^{-\frac32})\\
                &= f(\alpha) + O(n^{-\frac32}).
\end{aligned}
\]
As a caveat, we remark that the analysis above does not extend the error bound
$O(n^{-\frac32})$ to a type I error that is close to 0 or 1. The result states that the
approximation error of using the Edgeworth expansion quickly tends to 0 at the rate of
$O(n^{-\frac32})$. This error rate can be improved at the expense of a higher order
expansion of the Edgeworth series. For comparison, our analysis can be carried over to
show that the approximation error of using CLT is $O(n^{-\frac12})$.

\subsection{Computational Cost}
The computational cost of the Edgeworth approximation can be broken down as follows.
We first need to compute the cumulants of $L_i$ under $P_i$ and $Q_i$ up to a certain
order, for $i=1, \ldots, n$.
Next, we need to compute $F^{-1}_\ew(1 - \alpha)$. The
Cornish-Fisher approximation (Equation~\ref{eq:cornish_fisher}) costs constant time.
If we choose to compute the inverse numerically, we need to evaluate
Equation~\ref{eq:approx_Fn} then solve a one-dimensional root-finding problem. The
former has a constant cost, and the latter can be extremely efficiently computed by
various types of iterative methods, for which we can consider the cost as constant too.
Finally, it costs constant time to evaluate $\Ftilde_\ew(\cdot)$ using
Equation~\ref{eq:approx_Ftilde}.
Therefore, the only cost that might scale with the number of compositions $n$ is the
computation of cumulants.  However, in many applications including computing the privacy
bound for noisy SGD, all the $L_i$s are iid distributed. Under such condition, we only
need to compute the cumulants once. The total cost is thus a constant independent of
$n$.  This is verified by the runtime comparison in Section~\ref{sec:expr}.

%%% Local Variables:
%%% mode: latex
%%% TeX-master: "main"
%%% End:

\section{A Two-Parameter Privacy Interpreter}
\label{sec:param}
Let $M_1$ and $M_2$ be  two private algorithms that are associated with trade-off
functions $f_1$ and $f_2$, respectively. The algorithm $M_2$ will be more private than $M_1$ if $f_2$
upper bounds $f_1$. For the family of Gaussian differentially private algorithms,
this property can be reflected by the parameter $\mu \geq 0$ directly,
where a smaller value of $\mu$ manifests a more private algorithm.

Here we provide a two-parameter description $(\mu^*, \gamma)$ for the Edgeworth
approximation, through which the privacy guarantee between
two different approximations can be directly compared.

Given an approximation $f_\ew$, let $\alpha^*$ be its fixed point such that
$f_\ew(\alpha^*) = \alpha^*$.  Let $\mu^*$ be the parameter of GDP for which $G_{\mu^*}$
admits the same fixed point as $f_\ew$: $G_{\mu^*}(\alpha^*) = \alpha^*$.  Such $\mu^*$
can be computed in closed form:
\[ \mu^* = \Phi^{-1}(1 - \alpha^*) - \Phi^{-1}(\alpha^*). \]
Let $ \gamma = \int_{0}^1 f_\ew(\alpha) d\alpha$ be the area under the curve of $f_\ew$.

Two symmetric Edgeworth approximations\footnote{If $f$ is asymmetric, we can
always symmetrize it by taking $\min\set{f, f^{-1}}^{**}$.} $f_\ew^{(1)}$ and $f_\ew^{(2)}$ can be compared in the sense that
$f_\ew^{(2)}$ is more private than $f_\ew^{(1)}$ if their associated parameters
$({\mu^*}^{(1)}, \gamma^{(1)})$  and  $({\mu^*}^{(2)}, \gamma^{(2)})$ satisfy the
following condition:
\[
 {\mu^*}^{(2)} < {\mu^*}^{(1)} \; \text{and} \;  \gamma^{(2)} > \gamma^{(1)}.
\]

\begin{figure}[t]
    \centering
    \includegraphics[width=0.4\columnwidth]{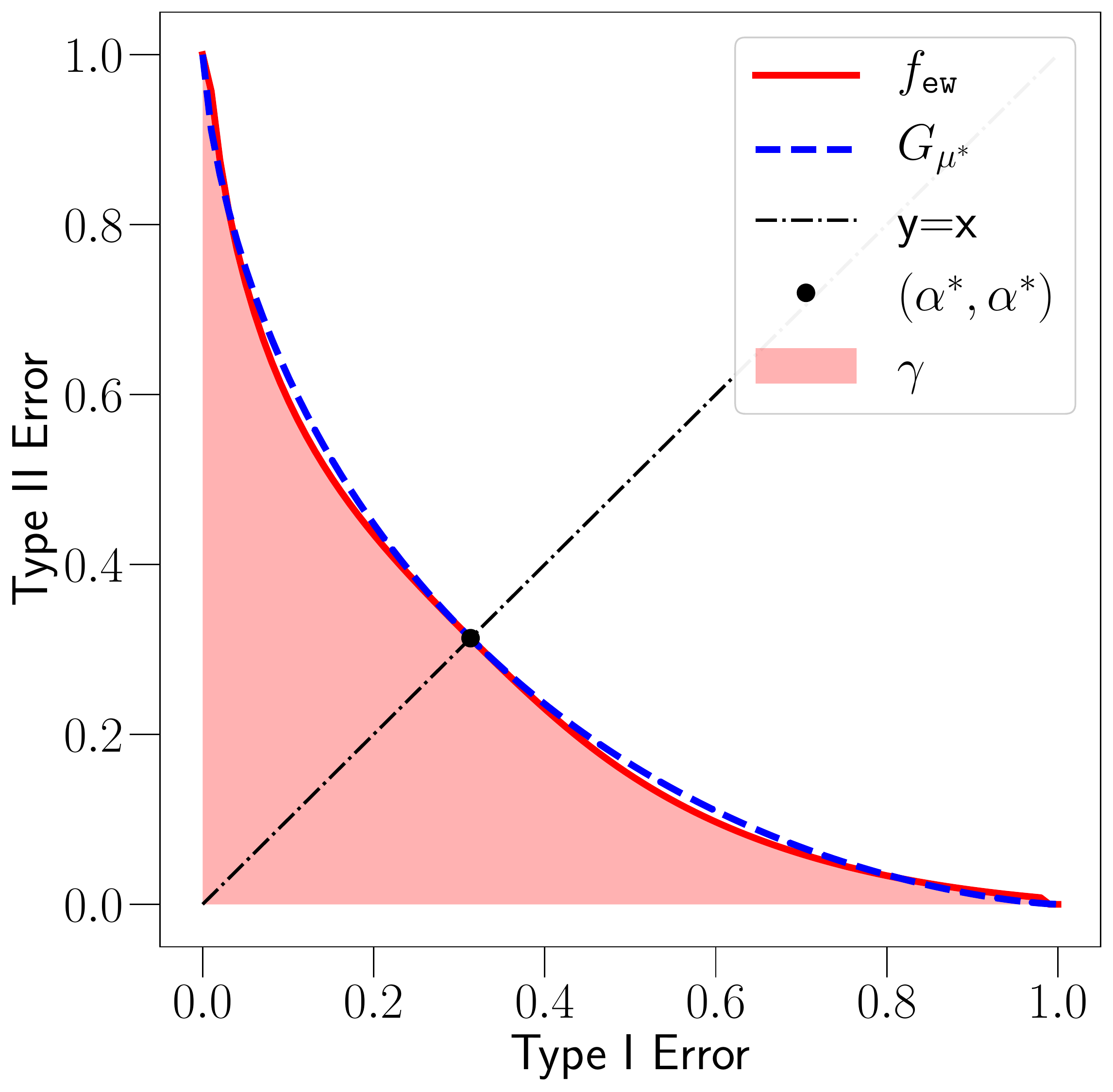}
    \includegraphics[width=0.4\columnwidth]{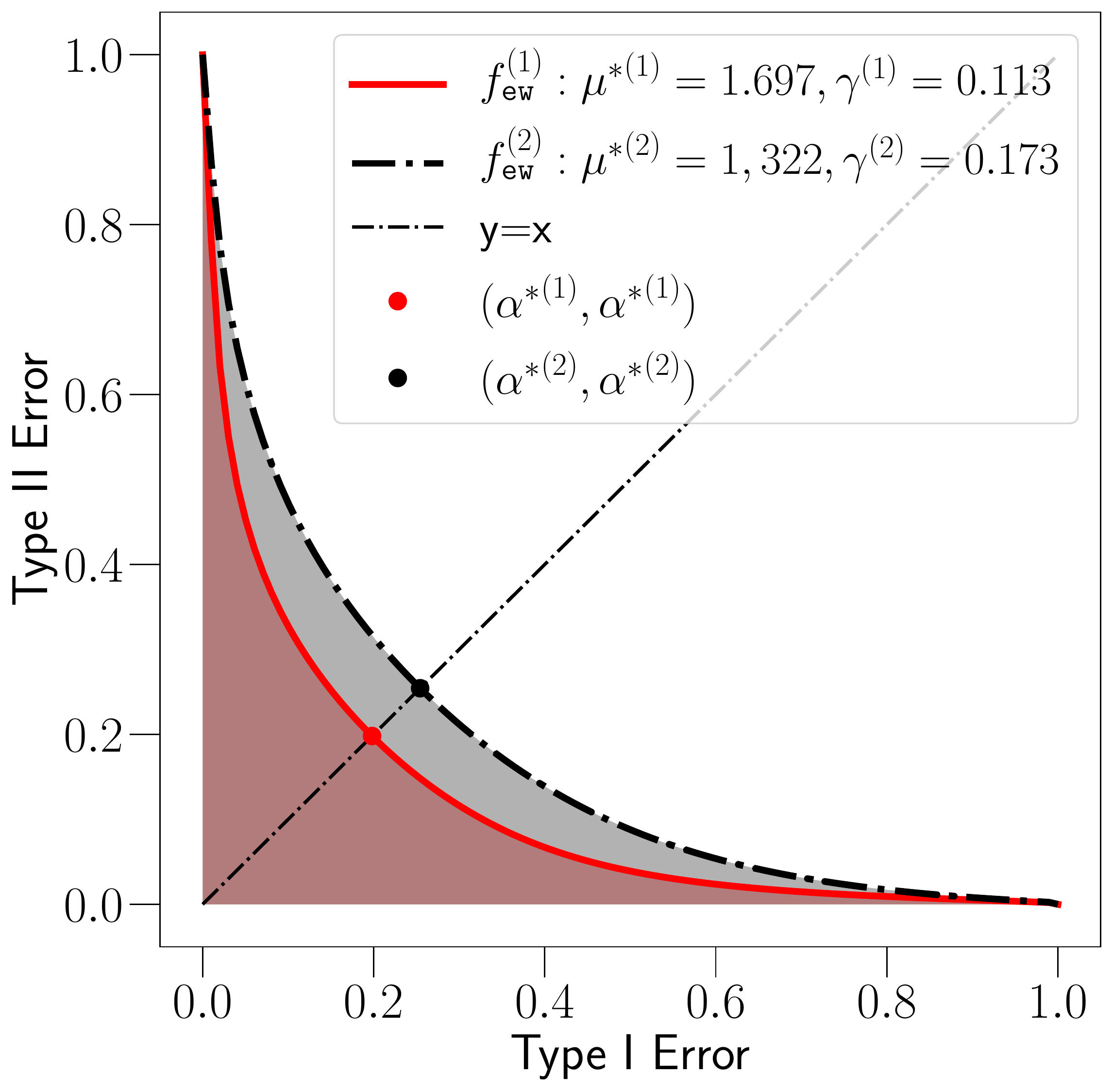}
    %\vskip-5pt\caption{Left: An illustration of the $(\mu^*, \gamma)$ parameterization. $f_\ew$, $G_{\mu^*}$ and the line $y=x$
    %intersect at point $(\alpha^*, \alpha^*)$.
    %Right: $f^{(2)}_\ew$ is more private than $f^{(1)}_\ew$.
    %Its intersection with
    %the line $y=x$ is further away from the original point than
    %$f^{(1)}_\ew$ (hence ${\mu^*}^{(2)} < {\mu^*}^{(1)} $),
    %and it has larger area under the curve (hence $\gamma^{(2)} > \gamma^{(1)}$).
    %}
    \vskip-5pt\caption{Left: An illustration of the $(\mu^*, \gamma)$ parameterization.
    Right: $f^{(2)}_\ew$ is more private than $f^{(1)}_\ew$.
    }
    \label{fig:ew_param}
\end{figure}

The left panel of Figure~\ref{fig:ew_param} provides the geometric interpretation of
the above parameterization. The Edgeworth approximation $f_\ew$, the CLT approximation
$G_{\mu^*}$, and the line $y=x$ intersect at the point $(\alpha^*, \alpha^*)$.

The right panel compares two Edgeworth approximations $f^{(1)}_\ew$ and $f^{(2)}_\ew$.
It is easy to see that in this case $f^{(2)}_\ew$ upper bounds $f^{(1)}_\ew$ and thus it is more private
than $f^{(1)}_\ew$. There are two important properties. First, its intersection with the line $y=x$ is further away from the
original point than $f^{(1)}_\ew$. Consider the geometric interpretation shown in the
left panel. This implies that ${\mu^*}^{(2)} < {\mu^*}^{(1)} $. Second, the
approximation $f^{(2)}_\ew$ also has a larger area under the curve than $f^{(1)}_\ew$,
which is essentially $\gamma^{(2)} > \gamma^{(1)}$.

This parameterization defines a partial order over the set
$\set{(\mu, \gamma): \; \mu \geq 0, \, 0 \leq \gamma \leq \frac12 }$\footnote{The
    perfect privacy is attained when the trade-off function is $\text{Id}(\alpha) = 1 -
\alpha$, whose area under the curve is $\frac12$.}.
It is also applicable to general trade-off functions.

\section{Experiments}
\label{sec:expr}
In this section, we present numerical experiments to compare the Edgeworth approximation
and the CLT approximation. Before we proceed, we introduce a numerical method to
directly compute the true composition bound in Section~\ref{sec:numerical}. This method
is not scalable and hence merely serves as a reference for our comparison. We use the
Edgeworth approximation of degree $2$ for all the experiments.  In the sequel, we refer
to those methods as \EW, \clt, and \num, respectively.  All the methods are implemented
in Python\footnote{The code is available at
    \url{https://github.com/enosair/gdp-edgeworth}.} and all the experiments are carried out on a MacBook with 2.5GHz processor and
16GB memory.

\subsection{A Numerical Method}
\label{sec:numerical}
Consider the problem of computing $f^{\otimes n}$ numerically. We know that we can find
$P,Q$ such that $f=T(P,Q)$ and $f^{\otimes n} = T(P^n,Q^n)$. However, computing
$T(P^n,Q^n)$ directly
involves high-dimensional testing, which can be challenging. We show this difficulty can
be avoided by going from the primal representation to the dual representation. Let $(\eps,
\delta_k(\eps))$ be the dual representation associated with $f^{\otimes k}$. The method
contains three steps to obtain $\delta_n(\eps)$ for $f^{\otimes n}$.

% 1) Convert $f$ to $\delta_1$. This step can be done implicitly via $P$ and $Q$, see Lemma~\ref{lem:delta_1}.\\
% 2) Iteratively compute $\delta_{k+1}$ from $\delta_k$ using Lemma~\ref{lem:delta_n}.\\
% 3) Convert $\delta_n$ to $f^{\otimes n}$ using Proposition~\ref{prop:dualprimal}.
\begin{enumerate}
    \item Convert $f$ to $\delta_1(\eps)$. This step can be done implicitly via $P$ and $Q$, see Lemma~\ref{lem:delta_1}.
    \item Iteratively compute $\delta_{k+1}(\eps)$ from $\delta_k(\eps)$ using Lemma~\ref{lem:delta_n}.
    \item Convert $\delta_n(\eps)$ to $f^{\otimes n}$ using Proposition~\ref{prop:dualprimal}.
\end{enumerate}

Next, we explain how to compute $\delta_{k+1}(\eps)$ from $\delta_k(\eps)$. First, we need a lemma that relates $\delta_1(\eps)$ with $P,Q$.

\begin{restatable}[]{lemma}{lemdelta}
\label{lem:delta_1}
	Suppose $f=T(P,Q)$ and $P,Q$ have densities $p,q$ with respect to a dominating measure $\mu$. Then the dual representation $\delta_1(\eps) = 1+f^*(-\e^\eps)$ satisfies
	$\displaystyle \delta_1(\eps) = \int (q-\e^\eps p)_+ \diff \mu.$
\end{restatable}
Suppose we are given $f_1 = T(P_1,Q_1), f_2 = T(P_2,Q_2)$. Let $\delta_1,\delta_2$ and $\delta_\otimes$ be the dual representations of $f_1, f_2$ and $f_1\otimes f_2$ respectively. The following lemma shows how to evaluate $\delta_{\otimes}$ from $\delta_1$ and $\delta_2$. To simplify notations, we assume $P_i,Q_i$ are distributions on the real line and have corresponding densities $p_i,q_i$ for $i=1,2$ with respect to Lebesgue measure. Generalization to abstract measurable space is straightforward.
\begin{restatable}[]{lemma}{deltan}
\label{lem:delta_n}
	 Let $L_2(y) = \log \frac{q_2(y)}{p_2(y)}$. Then
	 $\displaystyle \delta_{\otimes}(\eps) = \int_\R \delta_1\big(\eps-L_2(y)\big)q_2(y)\diff y.$
\end{restatable}
In particular, it yields a recursive formula to compute $f^{\otimes n}$ when $f=T(P,Q)$. Again we assume $P,Q$ has densities $p,q$ on the real line.
Let $L(x) = \log\tfrac{q(x)}{p(x)}$. We have
\begin{align*}
	\delta_0(\eps) &= \max\{1-\e^\eps,0\}\\
	\delta_{k+1}(\eps) &= \int\delta_k\big(\eps-L(x)\big)\diff x\\
	f^{\otimes n}(\alpha) &= \sup_{\eps\in
	\R} 1-\delta_n(\eps)-\e^\eps \alpha.
\end{align*}
We remark here that if $f$ is asymmetric, then the dual involves negative $\varepsilon$,
which is why the conversion to $f^{\otimes n}$ involves the whole real line.
The proof of the above lemmas is deferred to
Appendix~\ref{sec:details_of_the_numerical_method}.

In practice, it is more efficient to store $\delta_n$ in the memory than to perform the
computation on the fly, so we need to discretize the domain and store the function value on this grid. Consider an abstract grid $\{\eps_j\}_{j=1}^N\subseteq \R$, the recursion goes as follows:
\begin{align*}
	\delta_0(\eps_j) &= \max\{1-\e^{\eps_j},0\}, j=1,\ldots,N\\
	\delta_{k+1}(\eps_j) &= \int\delta_k\big(\lfloor\eps_j-L(x)\rfloor\big)\diff x, j=1,\ldots,N\\
	f^{\otimes n}(\alpha) &= \sup_{1\leqslant j\leqslant N} 1-\delta_n(\eps_j)-\e^{\eps_j} \alpha.
\end{align*}
where $\lfloor\eps_j-L(x)\rfloor = \max\{\eps_l:\eps_l\leqslant\eps_j-L(x)\}$. This rounding step can be replaced by an interpolation as well.

The major challenge in making this numerical method practical for computing composition
product of trade-off functions is that it is slow in computation as it involves $nN$
numerical integrations.

\subsection{A Moderate Number of Compositions}
Section~\ref{sec:edgeworth} shows that the approximation error of \EW is
$O(n^{-\frac32})$, and for \clt the error is $O(n^{-\frac12})$.  We thus expect for
small or moderate values of $n$, \EW will produce non-negligible outperformance to \clt. To verify this, we
investigate their performance on a toy problem for testing order-$n$ compositions of
Laplace distributions\footnote{ The density function of $\texttt{Lap}(\theta,b)$ is
$\frac{1}{2b} \exp(-\abs{x-\theta}/b)$. }: $\bP = \texttt{Lap}(0,1)^{\otimes n}$ vs $\bQ =
\texttt{Lap}(\theta,1)^{\otimes n}$.

\begin{figure}[tb]
\centering
    \includegraphics[width=0.7\textwidth]{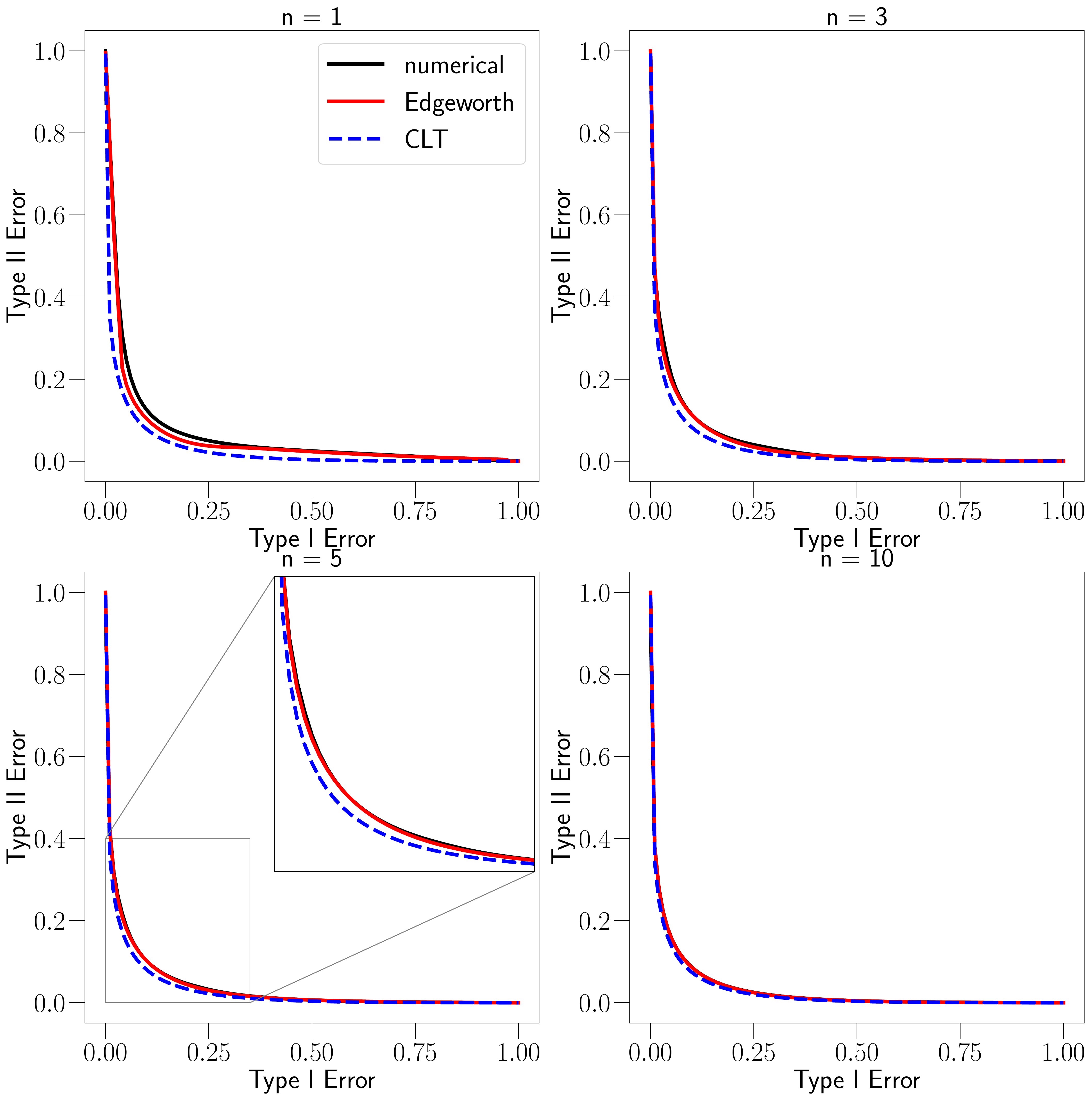}
    \vskip-5pt
    \caption{The estimated trade-off functions for testing
    $\texttt{Lap}(0, 1)^{\otimes n}$ vs $\texttt{Lap}(3/\sqrt{n}, 1)^{\otimes n}$.}
    \label{fig:laplace}
\end{figure}

\begin{figure}[tb]
\centering
    \includegraphics[width=0.7\textwidth]{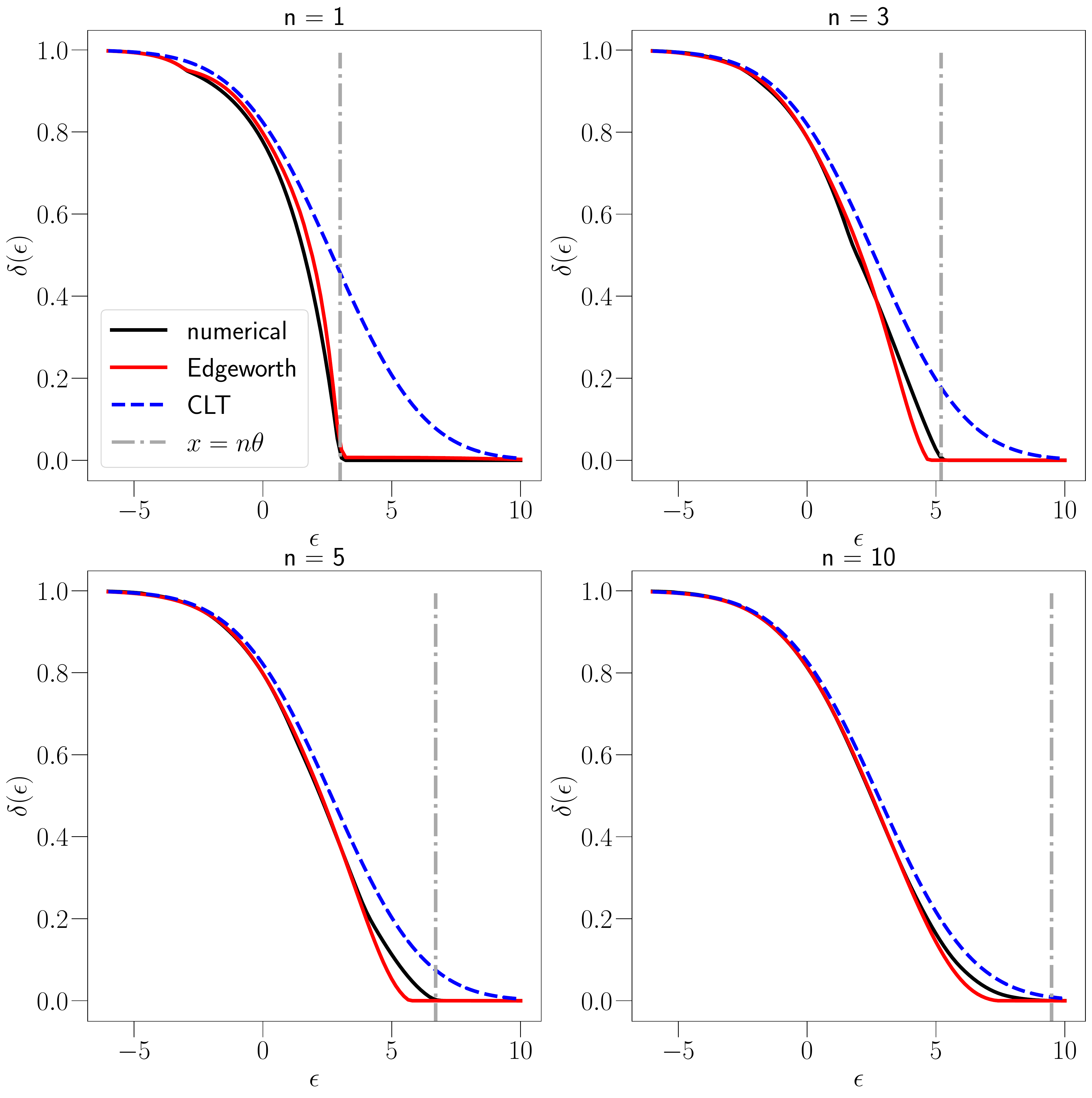}
    \vskip-5pt
    \caption{The associated $(\eps, \delta(\eps))$-DP of the estimated trade-off functions for testing % composed Laplace distributions:
    $\texttt{Lap}(0, 1)^{\otimes n}$ vs $\texttt{Lap}(3/\sqrt{n}, 1)^{\otimes n}$.}
    \label{fig:laplace_dual}
\end{figure}

We let the number of compositions $n$ vary from $1$ to $10$.  Since the privacy
guarantee decays as $n$ increases and the resulting curves would be very close to the
axes, we set $\theta = 3/\sqrt{n}$ for the sake of better visibility.
Figure~\ref{fig:laplace} plots the estimated trade-off functions for four representative
cases $n=1, 3, 5, 10$.
% The results for the other cases are similar.
For each of the methods, we also compute the associated $(\eps, \delta)$-DP (see
Proposition~\ref{prop:primaldual}) and plot $\delta$ as a function of $\eps$ in
Figure~\ref{fig:laplace_dual}. From both the primal and dual views, \EW coincides
better with \num in all the cases.  When the number of compositions is 10, even though
the difference between \EW and \clt is small in the primal visualization
(Figure~\ref{fig:laplace}), the $(\eps, \delta)$ presentation still clearly
distinguishes them.  In addition, due to the heavy tail of the Laplace distribution, we
shall have $ T(\set{\texttt{Lap}(0,1)}^{\otimes n}, \set{\texttt{Lap}(\theta,
1)}^{\otimes n}) \geq f_{\eps, 0}$ for $\eps \geq n\theta$ (see
Definition~\ref{prop:dualprimal} for the exact form of $f_{\eps, \delta}$).  Therefore,
the ground truth has the property that $\delta(\eps) = 0$ for $\eps \geq n\theta$.
Figure~\ref{fig:laplace_dual} shows
that \EW also outperforms \clt for predicting this changing point.

\begin{table}[h!]
\centering
%\scalebox{0.8}{
\begin{tabular}{c|c|c|c|c|c}
  \toprule & $n=2$ & $n=4$ & $n=6$ & $n=8$ & $n=10$ \\ \midrule
  CLT & 0.0004 & 0.0004 & 0.0004 & 0.0004 & 0.0005 \\ \hline
  Edgeworth & 0.2347 & 0.2341 & 0.2391 & 0.2222 & 0.2234 \\ \hline
  Numerical & 3.6834 & 7.3361 & 12.055 & 16.729 & 21.3575 \\ \bottomrule
\end{tabular}
%}
\vskip-5pt\caption{Runtime for estimating
the trade-off function for testing
$\texttt{Lap}(0,1)^{\otimes n}$ vs $\texttt{Lap}(3/\sqrt{n},1)^{\otimes n}$. }
\label{tbl:laplace_runtime}
\end{table}
Table~\ref{tbl:laplace_runtime} reports the runtime of the above experiment. \clt takes
constant computing time at the scale of 1e-4 second.  Due to the homogeneity of
the component distributions under composition, the runtime of \EW is also invariant of the
composition number, which is at the scale of 0.1 second. \num is computationally heavy.  Its
runtime is much larger and grows linearly as the number of compositions.

\begin{figure}[tb]
    \centering
    \includegraphics[width=0.4\columnwidth]{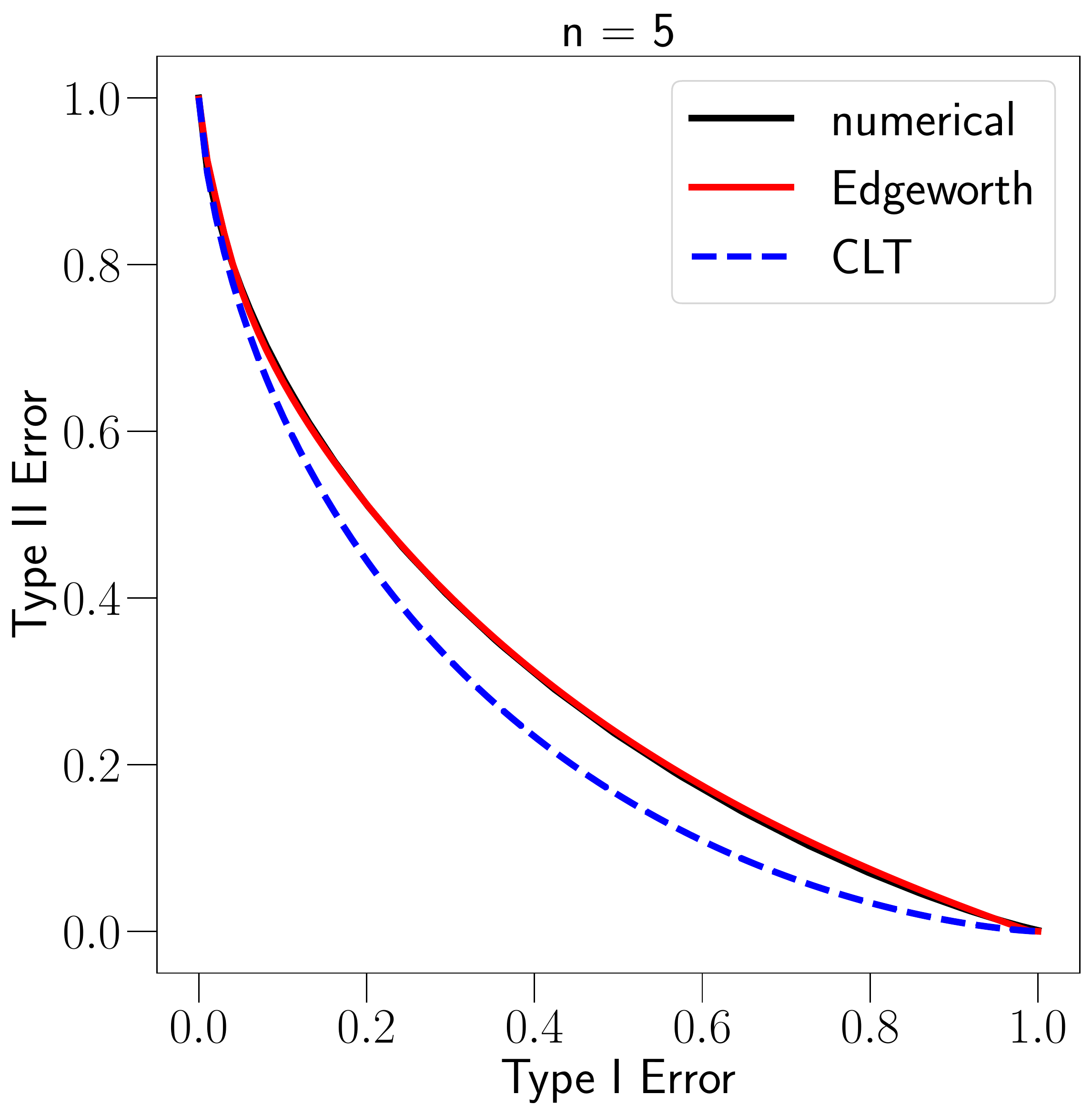}
    \includegraphics[width=0.4\columnwidth]{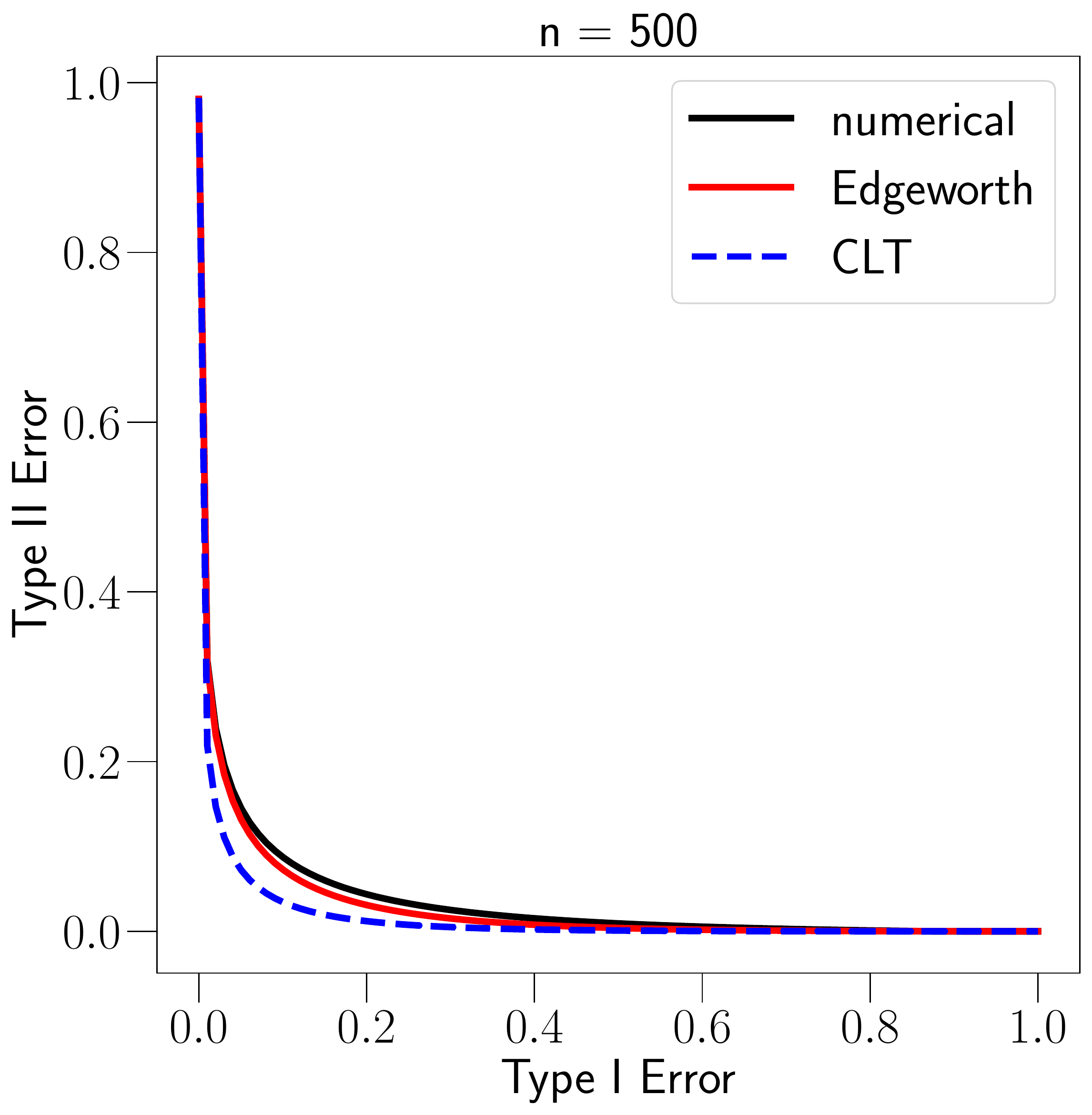}
    \vskip-5pt\caption{The estimation of $0.5/n^{\frac14}(G_1 + \text{Id})^{\otimes n}$.}
    \label{fig:mixture}
\end{figure}

\subsection{Privacy Guarantees for Noisy SGD}
\label{sec:expr_noisy_sgd}

\begin{figure}[tb!]
    \centering
    \includegraphics[width=0.32\columnwidth]{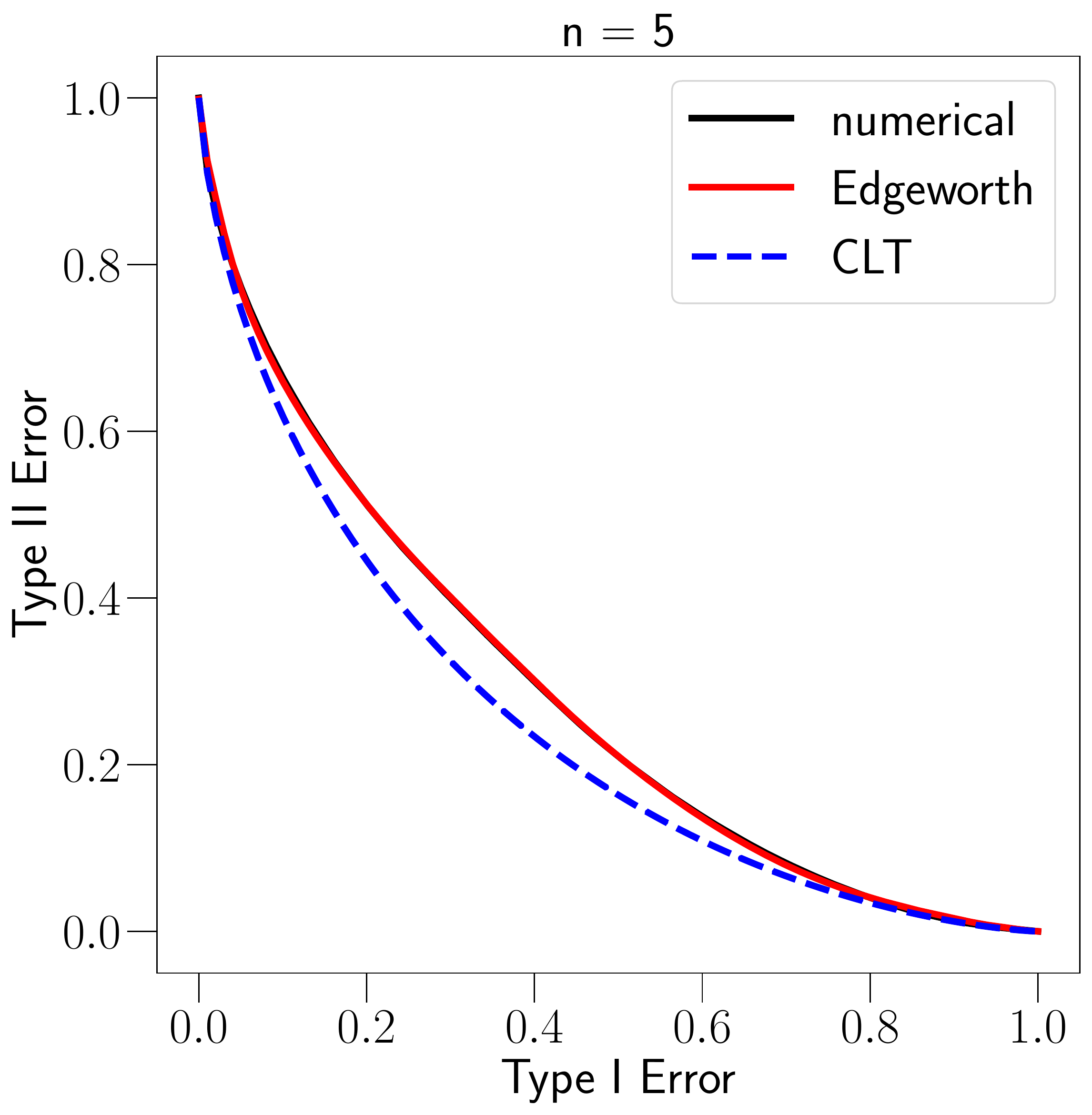}
    \includegraphics[width=0.32\columnwidth]{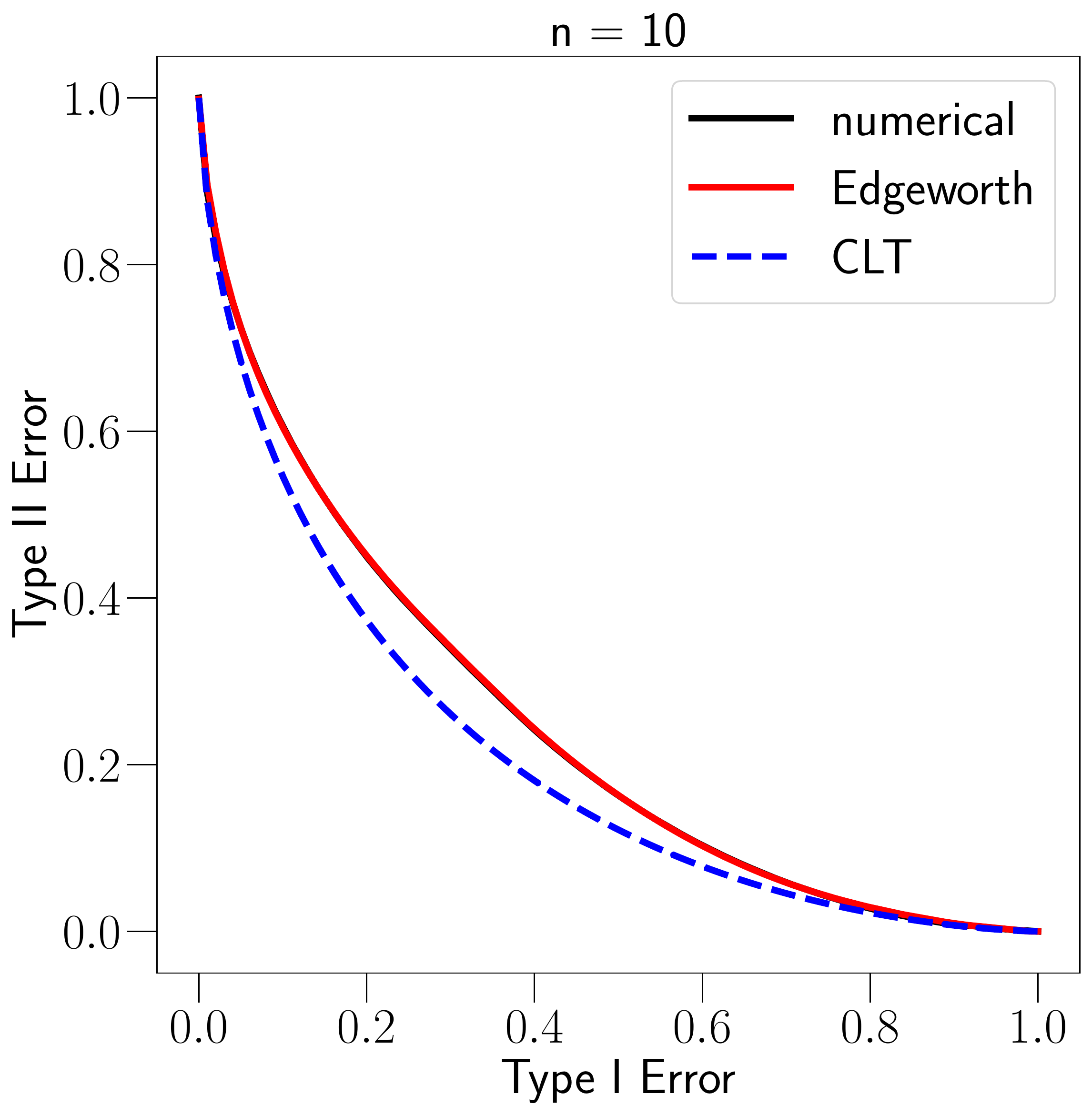}
    \includegraphics[width=0.32\columnwidth]{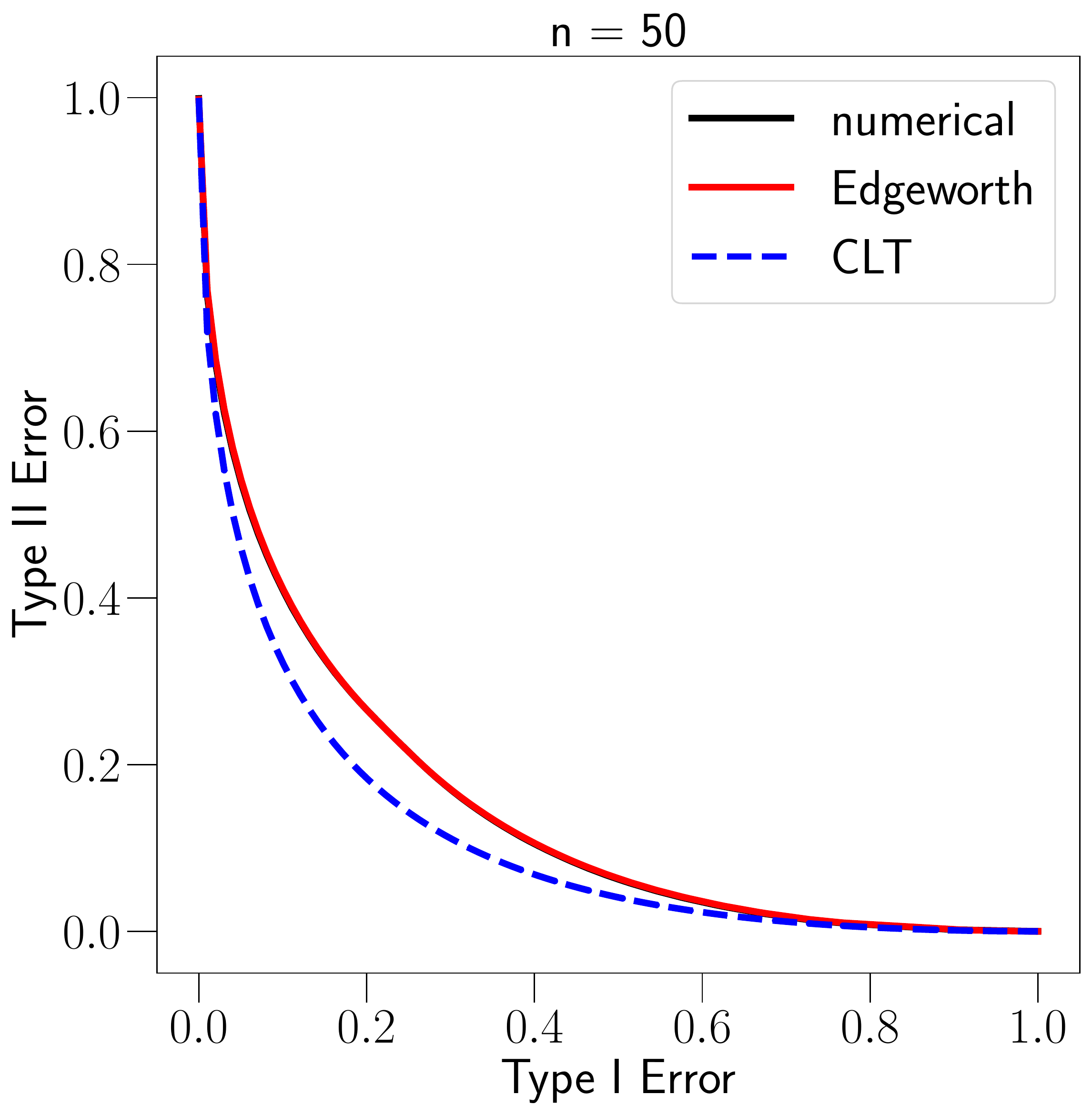}\\
    \includegraphics[width=0.32\columnwidth]{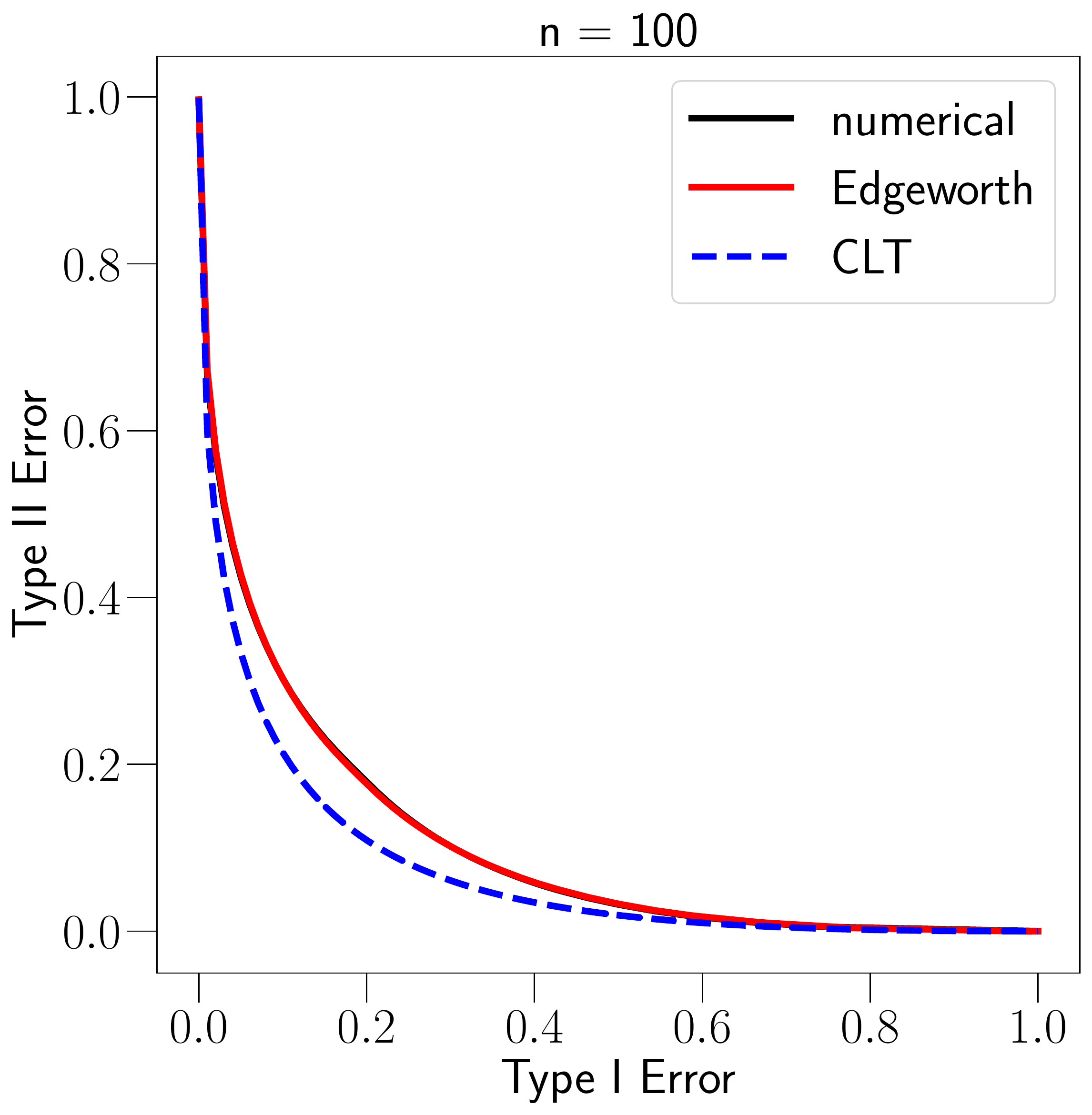}
    \includegraphics[width=0.32\columnwidth]{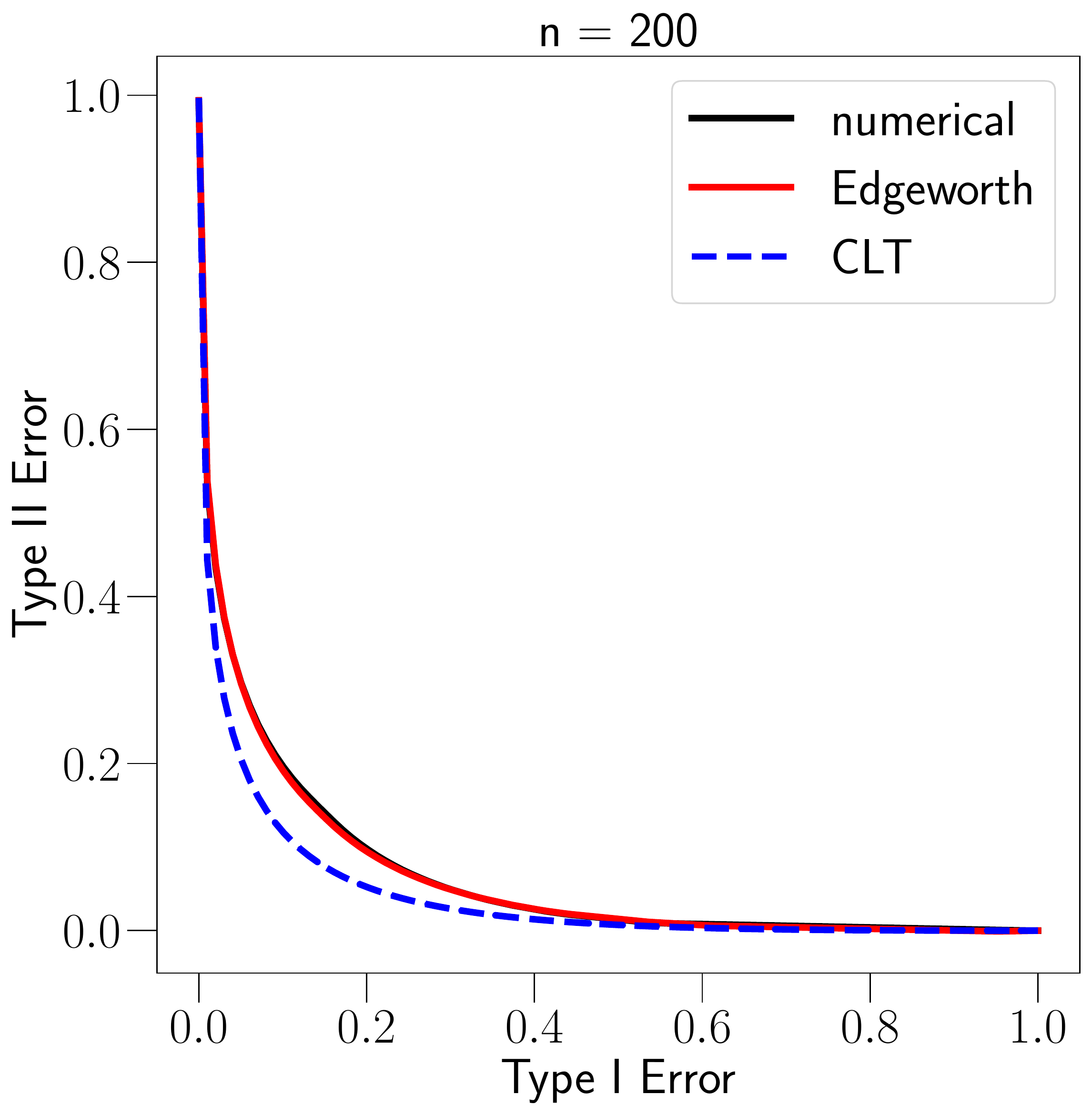}
    \includegraphics[width=0.32\columnwidth]{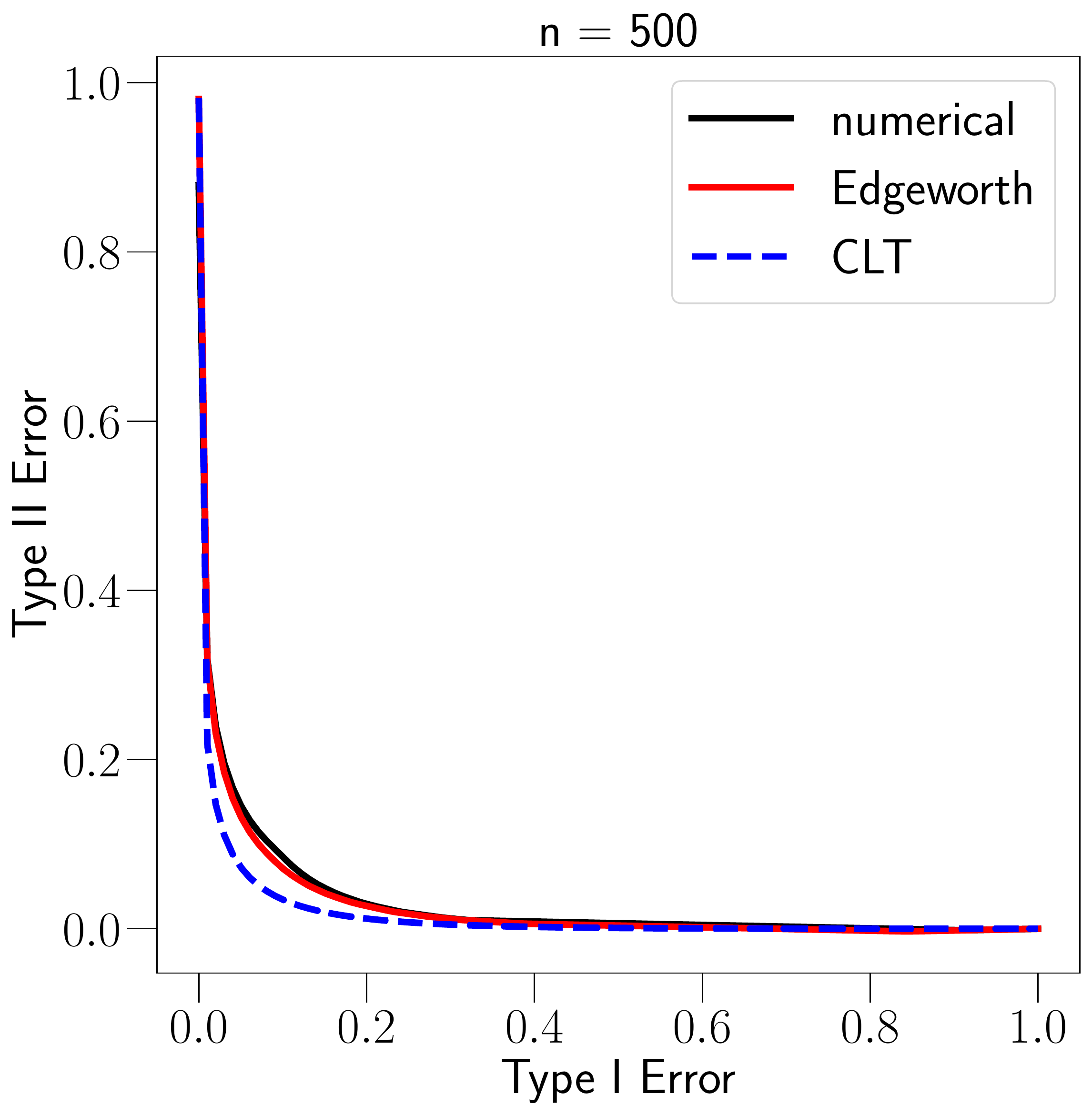}
    \vskip-5pt\caption{The estimation of the privacy bound for $n$-step noisy SGD.
    The sampling rate is $p=0.5/n^{\frac14}$ and the noise scale is $\sigma=1$.}
    \label{fig:mixture_symm}
\end{figure}
We inspect the performance of \EW and \clt on estimating the privacy guarantee for
$n$-step noisy SGD. As introduced in Section~\ref{sec:preliminary}, the privacy bound is
of form $\min\set{g, g^{-1}}^{**}$ where $g=(p G_{1/\sigma} + (1-p)\text{Id})^{\otimes n}$, and
the \clt estimation is $G_\mu$ with $\mu = p\sqrt{n(e^{1/\sigma^2} - 1)}$.  For \EW,
note that $pG_{1/\sigma} + (1-p)\text{Id}$ is the trade-off function of testing the standard
normal distribution $\N(0, 1)$ versus the mixture model $p\N(\frac1\sigma, 1) +
(1-p)\N(0,1)$ (see Appendix~\ref{app:mixture} for the proof).  It follows that $ g = T
(\N(0,1)^{\otimes n}, \set{  p\N(1/\sigma, 1) + (1-p)\N(0,1)  } ^{\otimes n} )$. As a
result, \EW can be applied by exploiting the cumulants of $\N(0, 1)$ and $p\N(1/\sigma, 1) +
(1-p)\N(0,1)$.

\begin{table}[tbh]
    \centering
    %\scalebox{0.85}{
    \begin{tabular}{c|c|c|c|c|c}
        \toprule & $n=1$ & $n=50$ & $n=100$ & $n=200$ & $n=500$\\ \midrule
        CLT & 0.0004 & 0.0004 & 0.0003 & 0.0003 & 0.0003\\ \hline
        Edgeworth & 0.2504 & 0.2198 & 0.2087 & 0.2257 & 0.3048\\ \hline
        Numerical & 35.7 & 2138.9 & 4980.7 & 7784.8 & 14212.1\\ \bottomrule
    \end{tabular}
    %}
    \vskip-5pt\caption{Runtime for computing $0.5/n^{\frac14}(G_1 + \text{Id})^{\otimes n}$.}
    \label{tbl:mixture_runtime}
\end{table}

%Again, for better visibility we set $p=0.5/n^{\frac14}$.
Here we present the result when $p=0.5/n^{\frac14}$. Since the convergence of \clt
requires the assumption $p\sqrt{n} \rightarrow \nu > 0$ \citep{bu2019deep},  this is the
regime that the performance of \clt does not have theoretical guarantees, and we shall
see in the following experiment that \clt significantly deviates \num.
Nevertheless, even in this case, \EW enjoys a nearly perfect match with \num.
We also investigated the case $p = 0.5/n^{-\frac{1}{2}}$, where the convergence of \clt
is guaranteed. In this regime, we shall see that \EW still outperforms
\clt. The results are deferred to Appendix~\ref{app:sgd_phalf}.

We let $n$ vary from 1 to 500. The noise scale $\sigma$ is set to 1.
Figure~\ref{fig:mixture} presents the estimated $g$ for two representative cases $n=5$
and $n=500$. Figure~\ref{fig:mixture_symm} shows the final estimation after
symmetrization, for a few selected values of $n$. Obviously, \EW is very close to \num
for all the cases.  Unlike the previous experiment, where \clt performs reasonably well
when $n$ reaches 10, here the difference between \clt and the other two remains
large even when $n=500$.

Table~\ref{tbl:mixture_runtime} presents the runtime comparison.  Clearly, \num is not
scalable.  It takes more than two hours to compute 500 compositions, whereas \clt and \EW
can be computed within one second.  The runtime of \num has grown three orders of
magnitude from $n=1$ to $n=500$, yet the runtime of \clt and \EW remains constant for
all the runs.

%%% optional
%\input{conclusion}

\subsection*{Acknowledgments}
We are grateful to Yicong Jiang for stimulating discussions in the early stages of this work. This work was supported in part by NSF through CAREER DMS-1847415, CCF-1763314, and CCF-1934876, the Wharton Dean's Research Fund, and NIH through R01GM124111.

%\clearpage
\bibliographystyle{alpha}
\bibliography{main}

\clearpage
\appendix
% !TEX root = main.tex
\appendix
\section{The Edgeworth Approximation}
\label{app:edgeworth}
We can apply Edgeworth expansion to approximate $\Ftilde_n$ directly, following the techniques introduced in
\cite{hall2013bootstrap}.  Let us assume $\bx \sim \bQ$. Denote
\begin{equation}
    X_Q = \frac{T_n - \Eq[T_n]}{\sqrt{\var_\bQ(T_n)}} = \frac{\sum_{i=1}^n (L_i - \mu_i)}{\sqrt{\sum_{i=1}^n \sigma^2_i}},
\end{equation}
where $\mu_i$ and $\sigma_i^2$ are the mean and variance of $L_i$ under the distribution
$Q_i$. The characteristic function of $X_Q$ is
\[
    \chi_n(t) = \exp \left( \sum_{i=1}^\infty \ktilde_r(X_Q) \frac{ (it)^r}{r!} \right),
\]
where $\ktilde_r(X_Q)$ is the $r$-th cumulant of $X_Q$. Details of how to compute the
cumulants are summarized in Appendix~\ref{app:cumulant}. Let $\sigma_n = \sqrt{\sum_{i=1}^n
\sigma^2_i}$. Particularly we have
\begin{equation}
    \begin{aligned}
        \ktilde_1(X_Q) & = \Eq(X_Q) = 0, \\
        \ktilde_2(X_Q) & = \var_\bQ(X_Q) = 1, \\
                     & \vdots \\
        \ktilde_r(X_Q) & = \ktilde_r\left(\sigma_n^{-1} \sum_{i=1}^n (L_i - \mu_i)
        \right) \\
        % = \sigma_n^{-r} \sum_{i=1}^n \ktilde_r(L_i - \mu_i)
        & = \sigma_n^{-r} \sum_{i=1}^n \ktilde_r(L_i), \; \; \forall r > 2.
    \end{aligned}
\end{equation}
We will denote the sum of $n$ cumulants by $ \Ktilde_r = \sum_{i=1}^n \ktilde_r(L_i)$.
Under the series expansion of the exponential function, we will have
\begin{equation}
    \label{eq:characteristic_func}
    \begin{aligned}
         \chi_n(t)  = &
        \exp\left(-\frac{t^2}{2}\right) \exp \left( \sum_{r=3}^\infty \frac{\sigma_n^{-r}}{r!} \Ktilde_r (it)^r \right) \\
          \approx &
        \exp\left(-\frac{t^2}{2}\right) \exp \left( \sum_{r=3, 4} \frac{\sigma_n^{-r}}{r!} \Ktilde_r (it)^r \right) \\
         \approx &
        \exp\left(-\frac{t^2}{2}\right) \bigg( 1
            + \sigma_n^{-3} \cdot \overbrace{\frac{1}{6}  \Ktilde_3 (it)^3}^{r_1(it)} \\
            & + \sigma_n^{-4} \cdot \overbrace{\frac{1}{24} \Ktilde_4 (it)^4}^{r_2(it)}
            + \sigma_n^{-6} \cdot \overbrace{\frac{1}{72} \Ktilde_3^2 (it)^6}^{r_3(it)} \bigg). \\
    \end{aligned}
\end{equation}
Since $\chi_n(t) = \int e^{ith} d\Ftilde_n(h)$ and $e^{-t^2/2} = \int e^{ith} d\Phi(h)$, we can obtain the corresponding ``inverse'' expansion:
\begin{equation}
    \label{eq:ftilde_inter1}
    \Ftilde_n(h)  \approx  \Phi(h)  + \sigma_n^{-3} \cdot R_1(h) + \sigma_n^{-4} \cdot
    R_2(h) + \sigma_n^{-6} \cdot R_3(h),
\end{equation}
and $R_j(h)$ is a function whose Fourier-Stieljes transform equals $r_j(it)e^{-t^2/2}$:
\[
    \int_{-\infty}^{\infty} e^{ith} dR_j(h) = r_j(it) e^{-t^2/2}.
\]
Let $D$ denote the differential operator $d/dh$. We have
\[
    e^{-t^2/2} = (-it)^{-j} \int_{-\infty}^{\infty} e^{ith} d \left\{ D^j \Phi(h) \right\}
\]
and hence
\[
    \int_{-\infty}^{\infty} e^{ith} d \left\{ (-D)^j \Phi(h) \right\} = (it)^j e^{-t^2/2}.
\]
Let us interpret $r_j(-D)$ as a polynomial in $D$, we then obtain
\[
    \int_{-\infty}^{\infty} e^{ith} d \left\{ r_j(-D) \Phi(h) \right\} = r_j(it)
    e^{-t^2/2}.
\]
Consequently,
\begin{equation}
    \label{eq:ftilde_inter2}
    R_j(h) = r_j(-D) \phi(h).
\end{equation}
It is well known that for $j \geq 1$,
\begin{equation}
    \label{eq:ftilde_inter3}
    (-D)^j\Phi(h) = -He_{j-1}(h) \phi(h)
\end{equation}
and $He_j$s are the Hermite polynomials:
\begin{equation}
    \label{eq:ftilde_inter4}
    \begin{aligned}
        He_0(h) & = 1, \;\; \\
        He_1(h) & = h, \;\; \\
        He_2(h) & = h^2 - 1, \\
        He_3(h) & = h^3 - 3h, \;\; \\
        He_4(h) & = h^4 - 6h^2 + 3, \\
        He_5(h) & = h^5 - 10h^3 + 15h, \;\; \\
        He_6(h) & = h^6 - 15h^4 + 45 h^2 - 15, \\
        He_7(h) & = h^7 - 21h^5 + 105 h^3, \\
                & \ldots
    \end{aligned}
\end{equation}
Combine equations~\ref{eq:ftilde_inter1}, \ref{eq:ftilde_inter2},
\ref{eq:ftilde_inter3} and \ref{eq:ftilde_inter4} we can deduce
the final result:
\begin{equation}
    \label{app:approx_Ftilde}
    \begin{aligned}
        & \Ftilde_n(h) && \approx && \Phi(h)  + \sigma_n^{-3} \cdot -\frac{1}{6}  \Ktilde_3 (h^2 - 1) \phi(h) \\
        & && && + \sigma_n^{-4} \cdot -\frac{1}{24} \Ktilde_4 (h^3 - 3h) \phi(h) \\
        & && && + \sigma_n^{-6} \cdot -\frac{1}{72} \Ktilde_3^2 (h^5 - 10h^3 + 15h) \phi(h).
    \end{aligned}
\end{equation}

In \ref{eq:characteristic_func}, the truncation happens in both the second and third line.
In the second line, we truncated terms where $r \geq 5$. In the following line, we apply
the series expansion to the exponential function, and we stopped after taking
$t_1 := \sigma^{-3}_n \cdot \frac{1}{6} \Ktilde_3(it)^3$,
$ t_2: = \sigma^{-4}_n \cdot \frac{1}{24} \Ktilde_4(it)^4 $ and the square of $t_1$.

The error stems from truncating $r \geq 5$ terms in the second line  will be dominated by
$ \dfrac{1}{120} \sigma_n^{-5} \Ktilde_5 (it)^5$ in the series expansion.
The error stems from truncating the expansion of $r=3, 4$ terms in the following line will be dominated by
the square of $t_2$: $\sigma^{-8}_n \cdot \frac{1}{576} \Ktilde^2_4(it)^8 $.

Since all $L_i$'s are identically distributed, the cumulants of $L_1, \ldots L_n$ take
the same value for any fixed order. Therefore,
$\sigma_1 = \cdots = \sigma_n = \sigma$ and $\ktilde_r = \ktilde_r(L_1) = \cdots =
\ktilde_r(L_n)$. As a consequence, we have $\sigma_n = \sqrt{n} \sigma$ and $\Ktilde_r = n \ktilde_r$.
This leads to
\begin{equation}
    \begin{aligned}
        & \sigma^{-3}_n \cdot \Ktilde_3(it)^3 \sim n^{-1/2} (it)^3, \\
        & \sigma^{-4}_n \cdot \Ktilde_4(it)^4 \sim n^{-1}   (it)^4, \\
        & \sigma^{-6}_n \cdot \Ktilde^2_3(it)^6 \sim n^{-1} (it)^6, \\
        & \sigma^{-8}_n \cdot \Ktilde^2_4(it)^8 \sim n^{-2} (it)^8, \\
        & \sigma^{-5}_n \cdot \Ktilde_5(it)^5 \sim n^{-3/2} (it)^5. \\
    \end{aligned}
\end{equation}
Hence the error for approximating $\chi_n(t)$ is upper bounded by
$O\left(n^{-2} (it)^8 + n^{-3/2}(it)^5 \right)$.
Next, we connect the characteristic function to CDF $\Ftilde_n(h)$.
From equations~\ref{eq:ftilde_inter2} and \ref{eq:ftilde_inter3},
we know the error term will be transformed into
$O\left( n^{-2} He_7(h) + n^{-3/2}He_4(h) \right)$ as approximating
$\Ftilde_n(h)$, which is
$O\left( n^{-2} h^7 + n^{-3/2} h^3 \right)$.

\section{Computing Cumulants From Moments}
\label{app:cumulant}
The cumulants of a random variable $X$ are defined using the cumulant-generating
function $K(t)$. It is the natural logarithm of the moment-generating function:
\[
    K(t) = \log \E\left(e^{tX}\right),
\]
and the cumulants are the coefficients in the Taylor expansion of $K(t)$ about the
origin:
\[
    K(t) = \log \E\left(e^{tX}\right) = \sum_{r=0}^\infty \kappa_r t^r / r!.
\]
For any integer $r \geq 0$, the $r$-th order non-central moment of $X$ is $\mu_r = \E(X^r)$. Recall the Taylor expansion of the moment-generating function $M(t)$ about the origin
\[
    M(t) = \E\left(e^{tX}\right) = \sum_{r=0}^\infty \mu_r t^r / r!
    = \exp\left( K(t)\right).
\]
The cumulants can be recovered in terms of the moments and vice versa. In general,
\[
\kappa_r=\sum _{k=1}^{r}(-1)^{k-1}(k-1)!B_{r,k}(\mu_1,\ldots ,\mu_{r-k+1})
\]
where $B_{n,k}$ are Bell polynomials. The relationship between the first few cumulants and moments is as the following:
\[
\begin{aligned}
& \kappa_0 = 0,\\
& \kappa_1 = \mu_1,\\
& \kappa_2 = \mu_2 - \mu_1^2,\\
& \kappa_3 = \mu_3 - 3\mu_2 \mu_1 + 2\mu^3_1,\\
& \kappa_4 = \mu_4 - 4 \mu_3 \mu_1 - 3\mu^2_2 + 12 \mu_2 \mu_1^2 - 6 \mu^4_1.\\
\end{aligned}
\]

\section{$\N(0,1)$ vs $p\N(\mu,1) + (1-p)\N(0,1))$}
\label{app:mixture}
Let $P$ be the standard normal distribution $\N(0,1)$ and $Q$
be a mixture model $p \N(\mu, 1) + (1-p) \N(0,1)$ with $\mu\geqslant 0$.
We now show that
\begin{lemma}
    $$T(P,Q) = p G_\mu + (1-p) \mathrm{Id}.$$
\end{lemma}
\begin{proof}
    The likelihood ratio between $Q$ and $P$ is
    $$p\e^{-\frac{1}{2}(x-\mu)^2+\frac{1}{2}x^2}+1-p = p\e^{\mu x-\frac{1}{2}\mu^2}+1-p.$$
    Since $\mu\geqslant 0$, likelihood ratio tests are thresholding, i.e., $\{x:x>h\}$. The type I and type II errors are
    \begin{align*}
        \alpha &= P\{x:x>h\} = 1-\Phi(h),\\
        \beta &= Q\{x:x\leqslant h\} \\
        &= p\E_{x\sim\N(\mu,1)}[1_{\{x:x\leqslant h\}}]+(1-p)\E_{x\sim\N(0,1)}[1_{\{x:x\leqslant h\}}]\\
        &= p\Phi(h-\mu)+(1-p)\Phi(h).
    \end{align*}
    Inverting the first formula, we have $h=\Phi^{-1}(1-\alpha)$. So
    $$\beta = p\Phi(h-\mu)+(1-p)\Phi(h) = p\Phi(\Phi^{-1}(1-\alpha)-\mu)+(1-p)(1-\alpha)$$
    Making use of the known expression $G_\mu(\alpha)=\Phi\big(\Phi^{-1}(1-\alpha)-\mu\big)$ and $\mathrm{Id}(\alpha) = 1-\alpha$, we have
    $$T(P,Q)(\alpha) = \beta = p G_\mu(\alpha) + (1-p) \mathrm{Id}(\alpha).$$
\end{proof}

\section{Details of the Numerical Method} % (fold)
\label{sec:details_of_the_numerical_method}
%\lemdelta*
\subsection{Proof of Lemma~\ref{lem:delta_1}}
\begin{proof}
    By definition of convex conjugacy, $\delta\geqslant\delta_1(\eps)$ if and only if $f(x)\geqslant 1-\delta-\e^\eps x$ for all $x\in[0,1]$. Since $f=T(P,Q)$ characterizes optimal testing rules, $f(x)\geqslant 1-\delta-\e^\eps x$ for any $x\in[0,1]$ if and only if for any event $E$, $Q[E]\leqslant \e^\eps P[E]+\delta$. That is,
    \begin{align*}
        \delta_1(\eps) &= \min\{\delta: Q[E]\leqslant \e^\eps P[E]+\delta, \forall E\}\\
        &= \max_{E} Q[E]-\e^\eps P[E]\\
        &= \max_{E}\int_{E} \big[q(x)-\e^\eps p(x)\big] \diff \mu(x).
    \end{align*}
    Obviously, the maximum is attained at the event that the integrand being non-negative. That is, $E=\{x:q(x)-\e^\eps p(x)\geqslant 0\}$. Therefore,
    $$
        \delta_1(\eps) = \int \big(q-\e^\eps p\big)_+ \diff \mu.$$
\end{proof}

% Suppose we are given $f_1 = T(P_1,Q_1), f_2 = T(P_2,Q_2)$. Let $\delta_1,\delta_2$ and $\delta_\otimes$ be the dual representations of $f_1, f_2$ and $f_1\otimes f_2$ respectively. Assuming $P_i,Q_i$ are distributions on the real line and have corresponding densities $p_i,q_i$ for $i=1,2$ with respect to Lebesgue measure, we have
%
\subsection{Proof of Lemma~\ref{lem:delta_n}}
\begin{proof}
    By definition of $\otimes$ and Lemma \ref{lem:delta_1}, we have
    \begin{align*}
        \delta_{\otimes}(\eps) &= 1+(f_1\otimes f_2)^*(-\e^\eps) \\
        &= 1+\big(T(P_1\times P_2,Q_1\times Q_2)\big)^*(-\e^\eps)
        && (\text{Def of } \otimes)\\
        &= \iint \big(q_1(x)q_2(y)-\e^\eps p_1(x)p_2(y)\big)_+ \diff x \diff y
        && (\text{Lemma } \ref{lem:delta_1})\\
        &=\iint q_2(y)\cdot\Big(q_1(x)-\e^\eps p_1(x)\cdot\tfrac{p_2(y)}{q_2(y)}\Big)_+ \diff x \diff y
        && (q_2(y)\geqslant 0)\\
        &=\iint q_2(y)\cdot\Big(q_1(x)-\e^{\eps-L_2(y)} p_1(x)\Big)_+ \diff x \diff y
        && (\text{Def of } L_2)\\
        &=\int q_2(y)\cdot \Big[\int \big(q_1(x)-\e^{\eps-L_2(y)} p_1(x)\big)_+ \diff x\Big] \diff y
        && (\text{Fubini})\\
        &= \int q_2(y)\cdot \delta_1\big(\eps-L_2(y)\big) \diff y.
        && (\text{Lemma }\ref{lem:delta_1}\text{ on }\delta_1)
    \end{align*}
\end{proof}
% section details_of_the_numerical_method (end)

\section{Privacy Guarantees for Noisy SGD with Sampling Rate $p=\frac{0.5}{\sqrt{n}}$}
\begin{figure}[tbh]
\label{app:sgd_phalf}
    \centering
    \includegraphics[width=0.24\columnwidth]{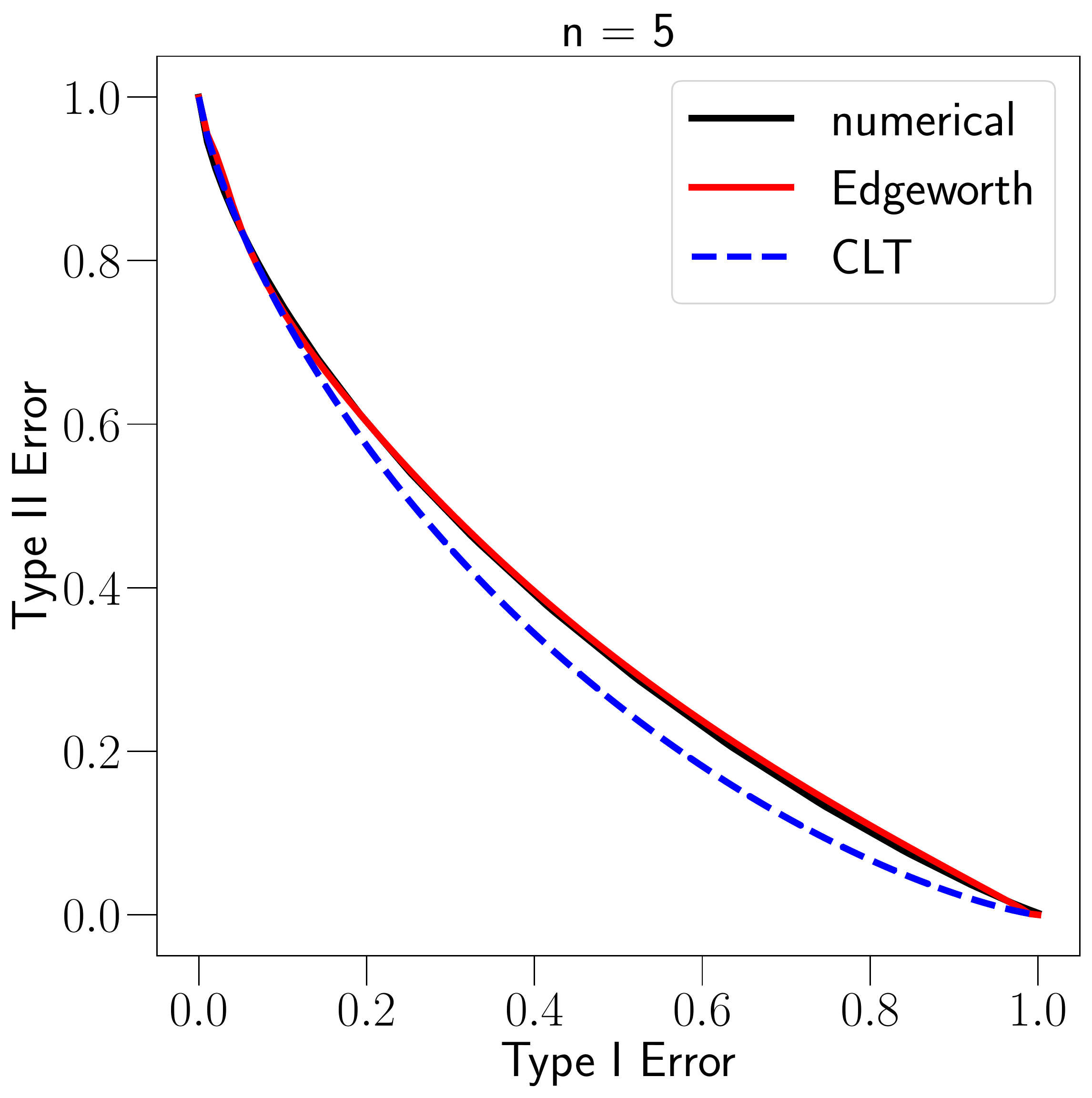}
    \includegraphics[width=0.24\columnwidth]{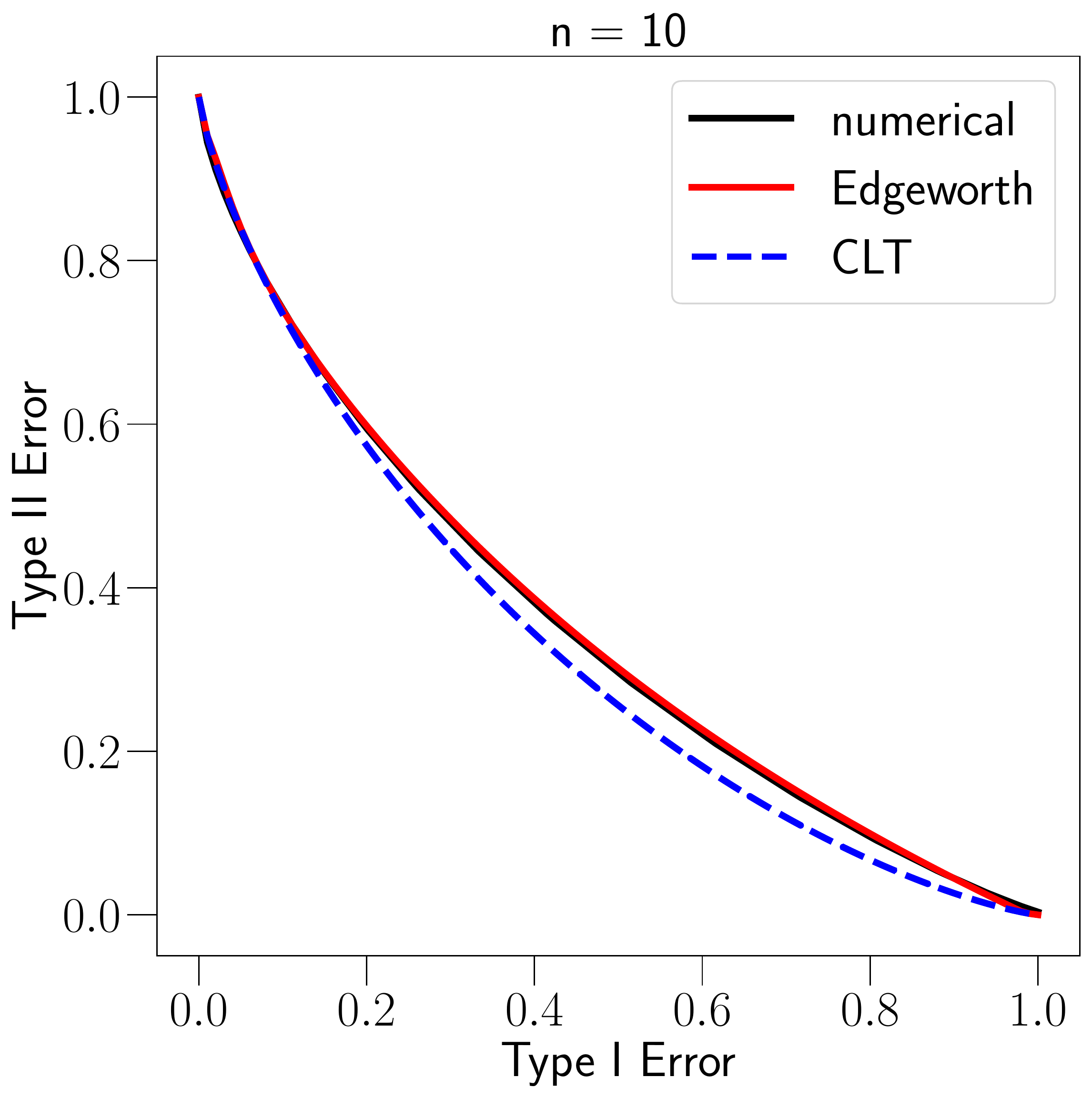}
    \includegraphics[width=0.235\columnwidth]{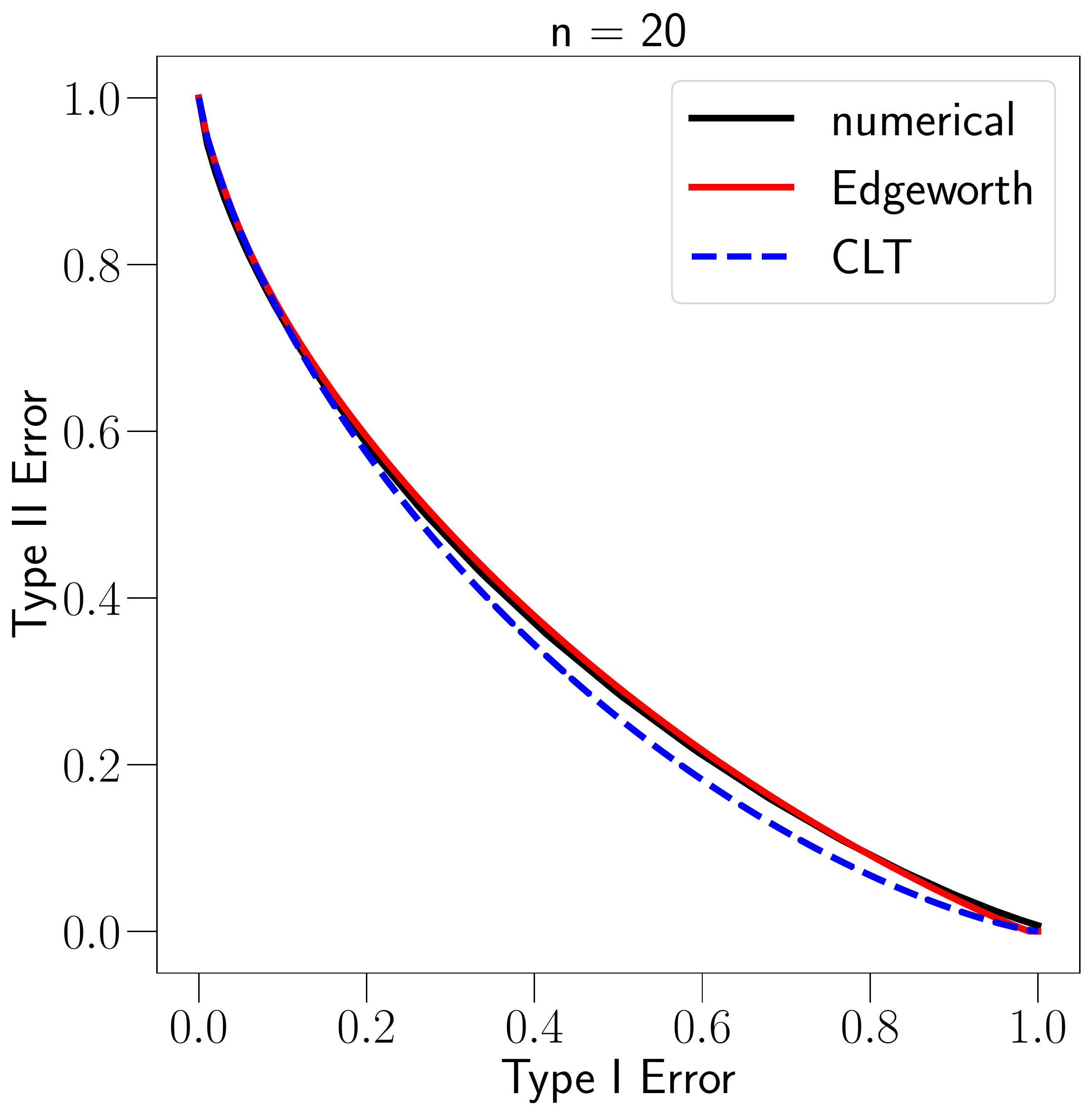}
    \includegraphics[width=0.235\columnwidth]{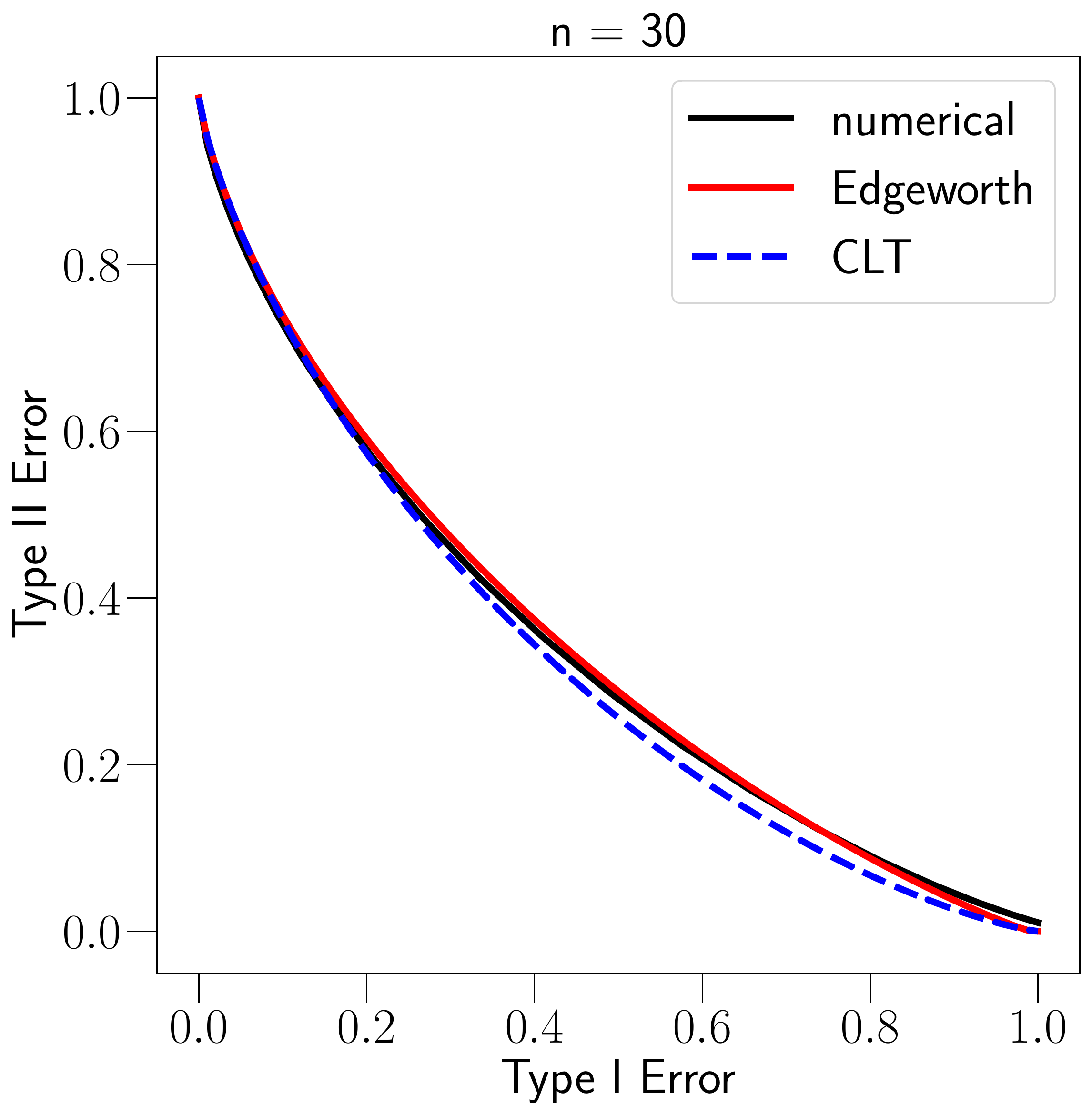}
    \vskip-5pt\caption{The estimation of $0.5/n^{\frac12}(G_1 + \text{Id})^{\otimes n}$.}
    \label{fig:mixture_phalf}
    \includegraphics[width=0.24\columnwidth]{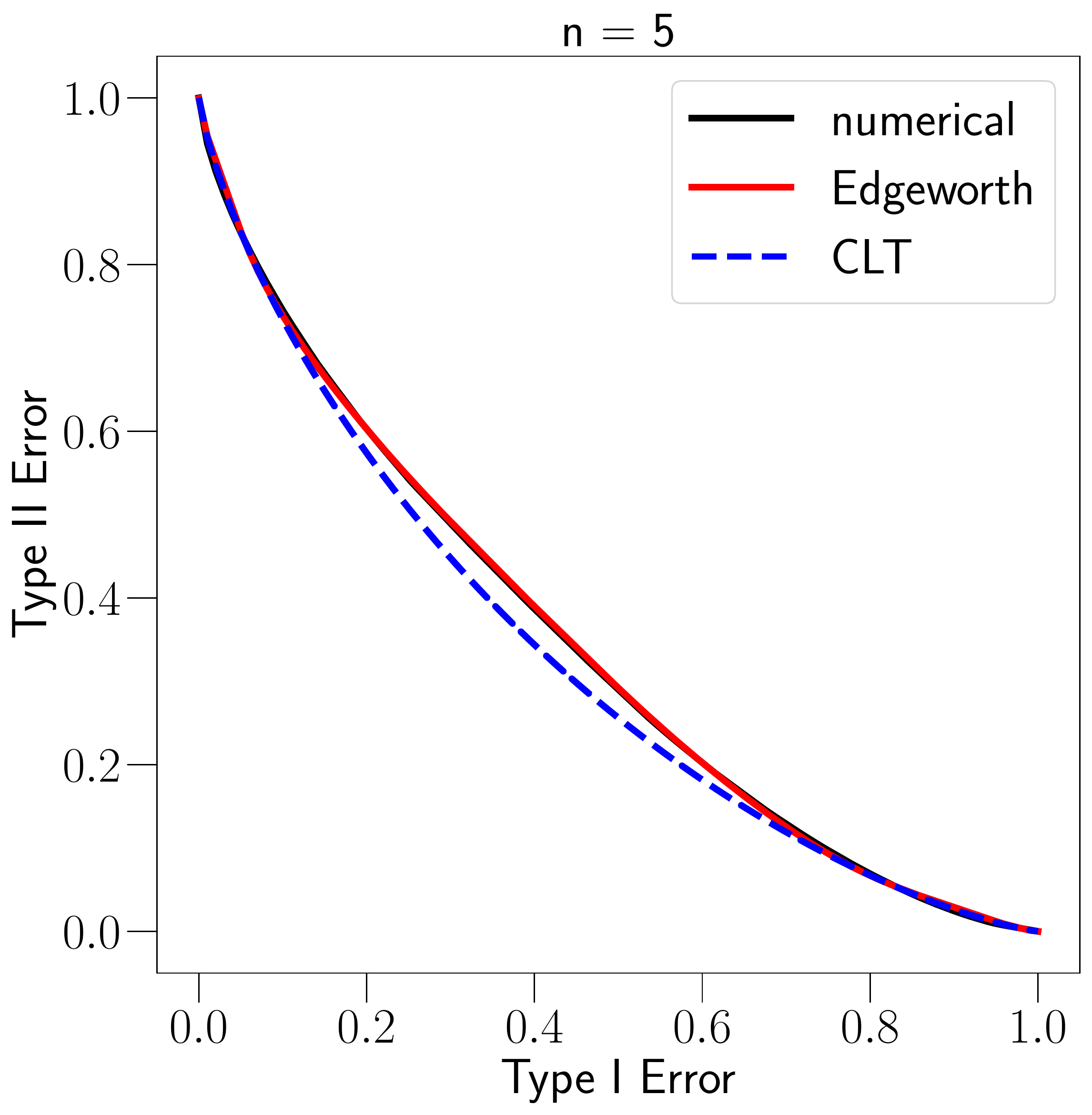}
    \includegraphics[width=0.24\columnwidth]{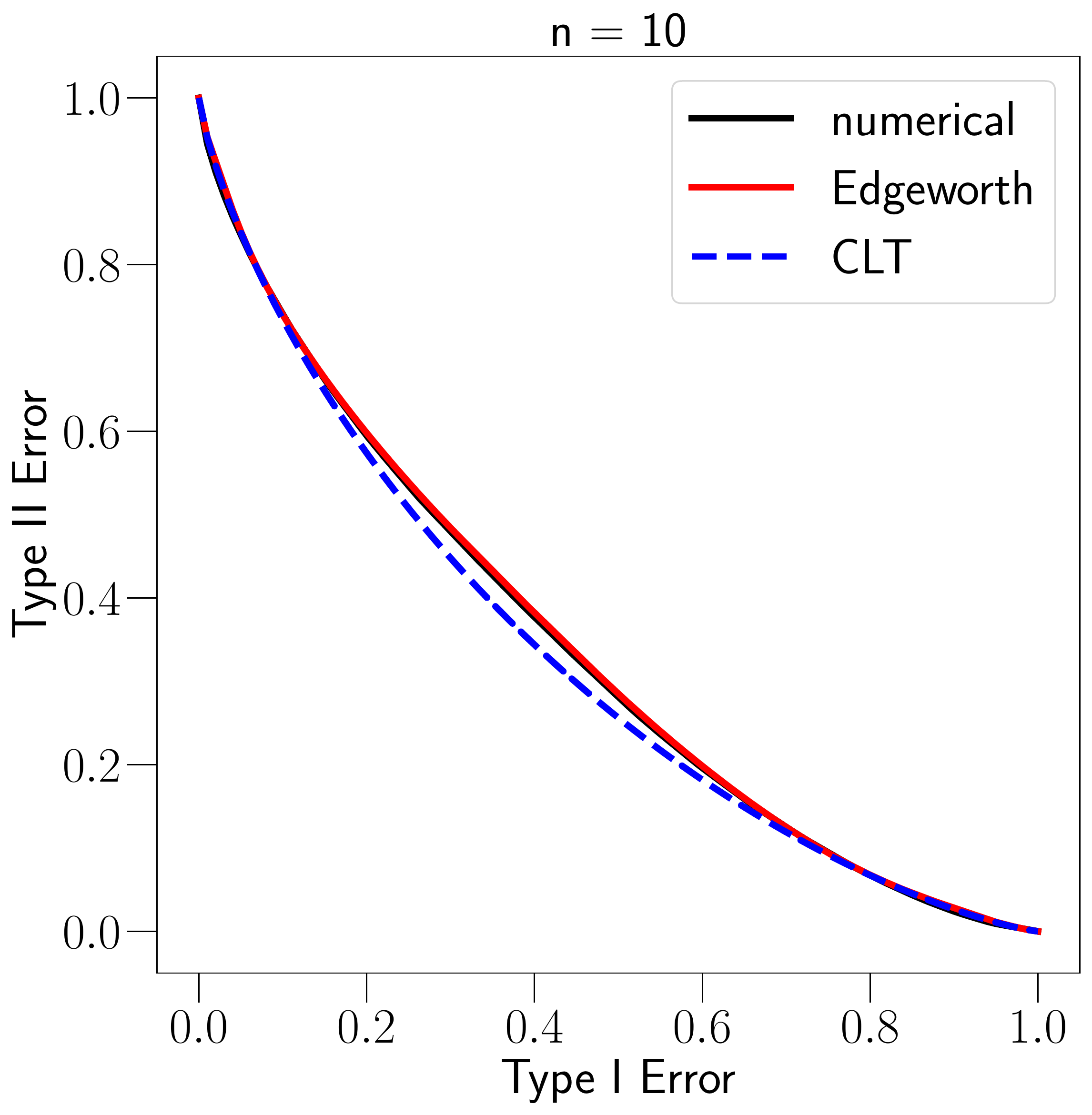}
    \includegraphics[width=0.24\columnwidth]{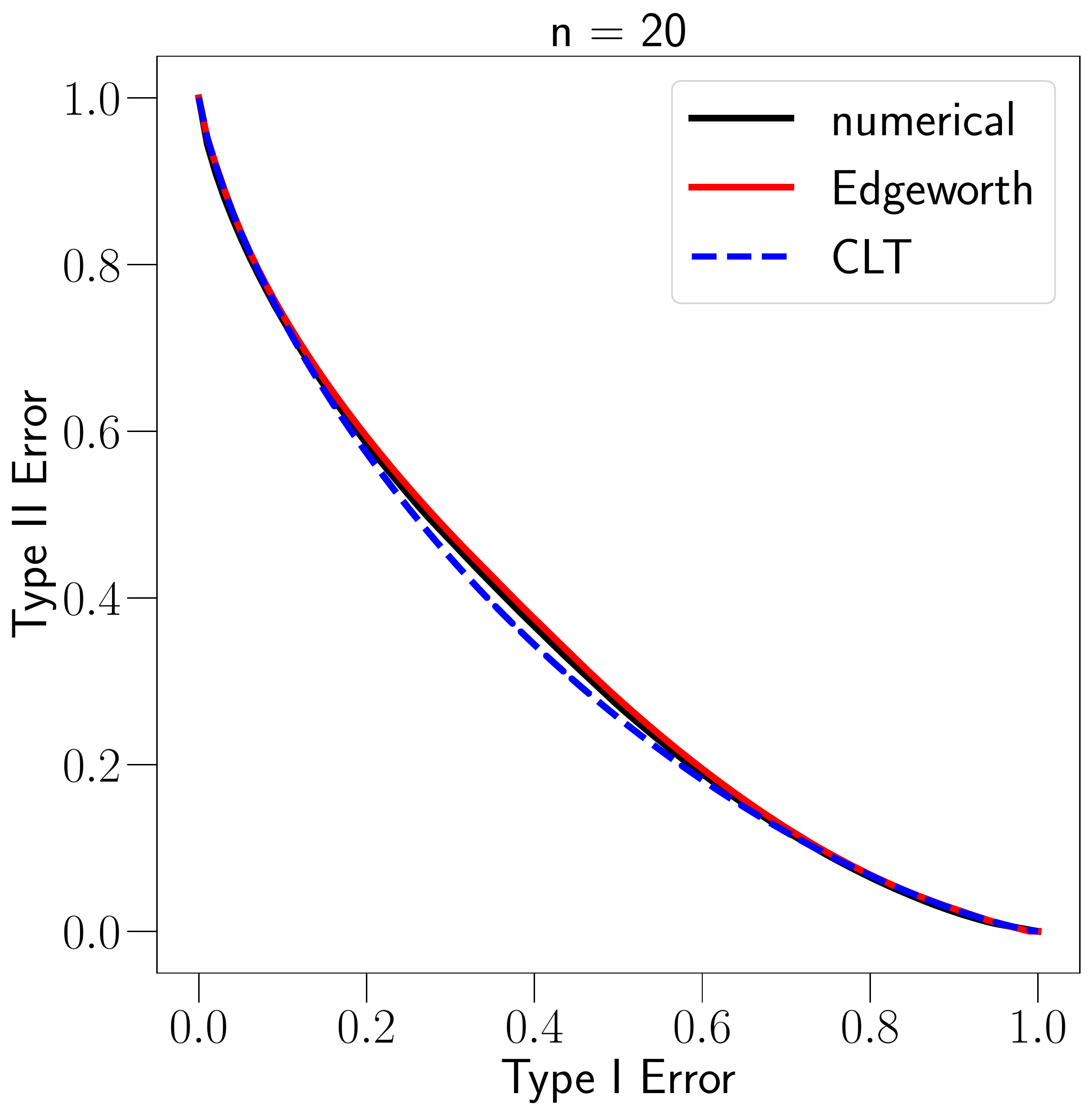}
    \includegraphics[width=0.24\columnwidth]{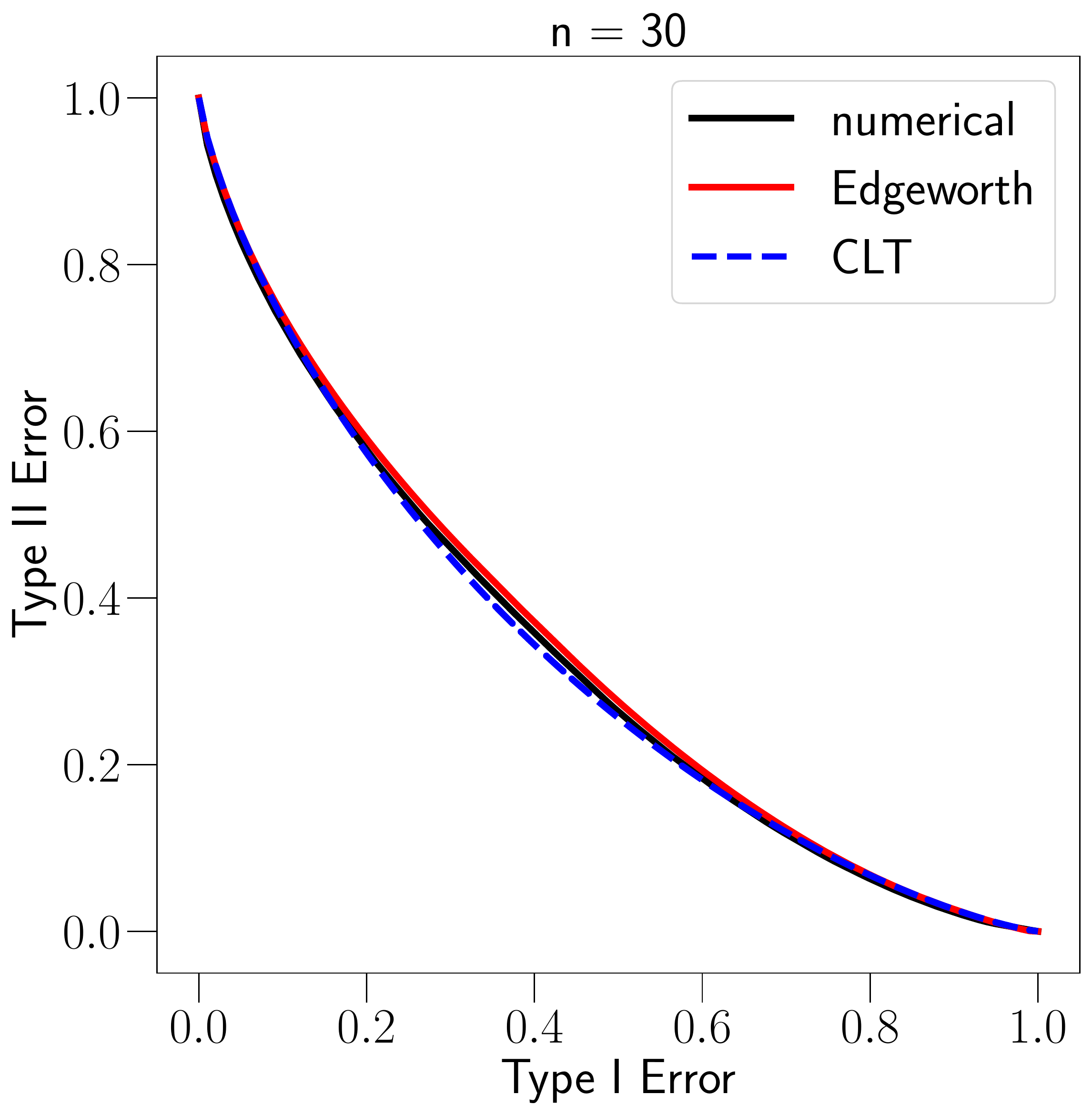}
    \vskip-5pt\caption{The estimation of the privacy bound for $n$-step noisy SGD.
    The sampling rate is $p=0.5/n^{\frac12}$ and the noise scale is $\sigma=1$.}
    \label{fig:mixture_symm_phalf}
\end{figure}

In Section~\ref{sec:expr_noisy_sgd} we present the result when the sampling rate
$p=0.5/n^{\frac14}$. Since the convergence of \clt requires the assumption $p\sqrt{n}
\rightarrow \nu > 0$ \citep{bu2019deep}, that is a regime where the performance of \clt
does not have theoretical guarantees. Here we present the results when $p =
0.5/n^{\frac12}$, where the convergence of \clt is guaranteed. However, we still
observe that \EW outperforms \clt.  See Figure~\ref{fig:mixture_phalf} and
\ref{fig:mixture_symm_phalf} for the comparison.

\end{document}